\documentclass[11pt,a4paper]{article}
\usepackage{a4wide}
\usepackage{graphicx} % more modern
\usepackage{epsfig} % less modern
\usepackage{subfigure} 
\usepackage{algorithm}
\usepackage{algorithmic}

\usepackage{comment}

\usepackage[hyphens]{url}
\usepackage{graphicx} % to include figur\usepackage{kasai}es
\usepackage{epsfig}
\usepackage{subfigure}

\usepackage{amssymb}
\usepackage{amsmath}
\usepackage{amsthm}
\newtheorem{assumption}{Assumption}[section]
\newtheorem{proposition}{Proposition}[section]
\newtheorem{theorem}{Theorem}[section]
\newtheorem{corollary}{Corollary}[section]
\newtheorem{lemma}{Lemma}[section]
\newtheorem{definition}{Definition}[section]

\usepackage{cite}

\def\vec#1{\mbox{\boldmath $#1$}}
\def\mat#1{\mbox{\bf #1}}
\newcommand{\gradf}{{\rm grad} f}

\newcommand{\changehss}[1]{#1}

\title{Riemannian stochastic variance reduced gradient with retraction and vector transport}
\date{\today}

\author{Hiroyuki Sato\thanks{Department of Applied Mathematics and Physics, Kyoto University, Kyoto, Japan ({\tt hsato@i.kyoto-u.ac.jp}).} \and Hiroyuki Kasai\thanks{Graduate School of Informatics and Engineering, The University of Electro-Communications, Tokyo, Japan ({\tt kasai@is.uec.ac.jp}).} \and Bamdev Mishra\thanks{\changehss{Microsoft, Hyderabad, India} ({\tt bamdevm@microsoft.com}.)}}

\begin{document}

\maketitle

\begin{abstract}
In recent years, stochastic variance reduction algorithms have attracted considerable attention for minimizing the average of a large but finite number of loss functions.
This paper proposes a novel Riemannian extension of the Euclidean stochastic variance reduced gradient (R-SVRG) algorithm to a manifold search space. The key challenges of averaging, adding, and subtracting multiple gradients are addressed with retraction and vector transport.
For the proposed algorithm, we present a global convergence analysis with a decaying step size as well as a local convergence rate analysis with a fixed step size under some natural assumptions.
In addition, the proposed algorithm is applied to 
the computation problem of the Riemannian centroid on the symmetric positive definite (SPD) manifold as well as the principal component analysis and low-rank matrix completion problems on the Grassmann manifold. The results show that the proposed algorithm outperforms the standard Riemannian stochastic gradient descent algorithm in each case\footnote{This paper extends the earlier work \cite{Kasai_arXiv_2016} to include more general results.}.
\end{abstract}

{\scriptsize Published in SIAM Journal on Optimization: \url{https://epubs.siam.org/doi/abs/10.1137/17M1116787}.}\\


\section{Introduction}
\label{Sec:intro}

A general loss minimization problem is defined as $\min_{w} f(w)$, where $f(w):= \frac{1}{N} \sum_{n=1}^N f_n(w)$, $w$ is the model variable, $N$ is the number of samples, and $f_n(w)$ is the loss incurred on the $n$-th sample. The {\it full gradient descent} (GD) algorithm requires the evaluation of $N$ derivatives, i.e., $\sum_{n=1}^N \nabla f_n(w)$, per iteration, which is computationally expensive when $N$ is extremely large. A well-known alternative uses only one derivative $\nabla f_n(w)$ per iteration for the $n$-th sample, and it forms the basis of the {\it stochastic gradient descent} (SGD) algorithm. When a relatively large step size is used in SGD, the training loss first decreases rapidly but results in large fluctuations around the solution. Conversely, when a small step size is used,
a large number of iterations are required for SGD to converge.
To circumvent this problem, SGD starts with a relatively large step size and gradually decreases it.

Recently, {\it variance reduction} techniques have been proposed to accelerate SGD convergence \cite{Defazio_NIPS_2014_s,Johnson_NIPS_2013_s, Mairal_SIAMJOPT_2015,Roux_NIPS_2012_s,Shalev_arXiv_2012_s,Shalev_JMLR_2013_s,Zhang_SIAMJO_2014_s}.
The stochastic variance reduced gradient (SVRG) algorithm is a popular technique with excellent convergence properties~\cite{Johnson_NIPS_2013_s}.
For smooth and strongly convex functions, SVRG has convergence rates similar to those of the stochastic dual coordinate ascent algorithm~\cite{Shalev_JMLR_2013_s} and the stochastic average gradient (SAG) algorithm~\cite{Roux_NIPS_2012_s}.
Garber and Hazan \cite{Garber_arXiv_2015_s} analyzed the convergence rate of SVRG when $f$ is a convex function that is the sum of nonconvex (but smooth) terms, and they applied their result to the principal component analysis (PCA) problem. Shalev--Shwartz  \cite{Shalev_arXiv_2015_s} also obtained similar results. Allen--Zhu and Yuan \cite{allen2016improved} further studied the same case with better convergence rates. Shamir \cite{Shamir_arXiv_2015_s} specifically studied the convergence properties of the variance reduction PCA algorithm.
More recently, Allen-Zhu and Hazan \cite{allen2016variance} and Reddi et al.~\cite{reddi2016stochastic} independently proposed variance reduction methods for faster nonconvex optimization.
However, it should be noted that all these cases assume a Euclidean search space.

In this paper, we handle problems in which the variables have a manifold structure:
\begin{equation}
\label{Prob:general}
\min_{w \in \mathcal{M}} f(w):=\frac{1}{N}\sum_{n=1}^N f_n(w),
\end{equation}
where $\mathcal{M}$ is a Riemannian manifold and $f_n, n=1,2,\dots, N$ are real-valued functions on $\mathcal{M}$.
These problems include, e.g., the low-rank matrix completion problem~\cite{Mishra_ICDC_2014_s}, the Riemannian centroid computation problem, and the PCA problem.
In all these problems, optimization on {\it Riemannian manifolds} has shown state-of-the-art performance. The Riemannian framework exploits the geometry of the search space, which is characterized by the constraints of the optimization problem.
Numerous efficient optimization algorithms have been developed~\cite{Absil_OptAlgMatManifold_2008}.
Specifically, the problem $\min_{w \in \mathcal{M}} f(w)$, where $\mathcal{M}$ is a Riemannian manifold, is solved as an \emph{unconstrained optimization problem} defined over the Riemannian manifold search space.
Furthermore, although these algorithms mainly address \textit{batch-based} approaches, Bonnabel~\cite{Bonnabel_IEEETAC_2013_s} proposed a {\it Riemannian stochastic gradient descent} (R-SGD) algorithm that extends SGD from Euclidean space to Riemannian manifolds.
Recently, more advanced stochastic optimization algorithms have also been generalized to Riemannian manifolds, including R-SQN-VR~\cite{pmlr-v84-kasai18a} and R-SRG~\cite{pmlr-v80-kasai18a}.

Building upon the work of Bonnabel \cite{Bonnabel_IEEETAC_2013_s}, we propose an extension of the stochastic variance reduced gradient algorithm to a Riemannian manifold search space (R-SVRG) and novel analyses.
This extension is nontrivial and requires particular consideration for handling the averaging, addition, and subtraction of multiple gradients at different points on the manifold $\mathcal{M}$.
Toward this end, this paper specifically leverages the notions of retraction and vector transport.
The algorithm and convergence analysis presented in this paper are generalized, which in itself is a challenging problem, in the retraction and vector transport case, as well as in the exponential mapping and parallel translation case,
allowing for extremely efficient implementation and making distinct contributions compared with an existing approach~\cite{Zhang_NIPS_2016} that relies only on the exponential mapping and parallel translation case.

It should be mentioned that a recent study~\cite{Zhang_NIPS_2016} by Zhang et al., which appeared simultaneously with our technical report \cite{Kasai_arXiv_2016}, has also proposed R-SVRG on manifolds. The main difference between our work and \cite{Zhang_NIPS_2016} is that we provide convergence analyses for the algorithm with retraction and vector transport, whereas \cite{Zhang_NIPS_2016} deals with a special case in which exponential mapping and parallel translation are used as retraction and vector transport, respectively.
There are additional differences.
Our convergence analysis handles global and local convergence analyses separately, as in the typical analyses of batch algorithms on Riemannian manifolds~\cite{Absil_OptAlgMatManifold_2008}.
Another difference is that our assumptions for the local convergence rate analysis are imposed only in a local neighborhood around a minimum, which is milder and more natural than the assumptions in \cite{Zhang_NIPS_2016}, which assumes Lipschitz smoothness in the entire space.
In other words, our global convergence analysis is not for a convergence rate or complexity, but it is an asymptotic convergence analysis.
Here, according to classical usage in nonlinear programming, we use the term global convergence for convergence to a critical point from any initial point.
On the other hand, our local convergence analysis is for a strongly convex function near the optimum, and the class of such functions includes many nonconvex functions that do not necessarily have global strong convexity.
Consequently, our analysis should be applicable to different types of manifolds.
For example, the parallel translation on the Stiefel manifold, which is an extremely important manifold in practice, is not available in a closed form.
We can use a vector transport based on the orthogonal projection to the tangent space of the Stiefel manifold as an efficient implementation.
Therefore, compared to the results of \cite{Zhang_NIPS_2016}, our convergence results with retraction and vector transport enable us to deal with a wider variety of manifolds.

We emphasize that although we derive a local convergence rate with retraction and vector transport in this paper,
it can be immediately used to derive a global iteration complexity
if we assume, e.g., strong convexity of $f$ globally on the search space.
Here, local convergence rate and global iteration complexity imply the rate at which the iterates approach a critical point in a sufficiently small neighborhood and the iteration complexity of obtaining a critical point from an arbitrary initial point, respectively.
The result thus obtained can be regarded as a generalization of the global iteration complexity with strongly convex $f$ obtained in \cite{Zhang_NIPS_2016}, where exponential mapping and parallel translation are used.

The remainder of this paper is organized as follows.
Section~\ref{Sec:GrassmannAndProblems} discusses Riemannian optimization theory, including the background on Riemannian geometry and some geometric tools used for optimization on Riemannian manifolds.
Section~\ref{Sec:R-SVRG} provides a detailed description of R-SVRG.
Sections \ref{Sec:GlobalAnalysis} and \ref{Sec:LocalAnalysis} present the global convergence analysis and local convergence rate analysis of the proposed R-SVRG, respectively.
Section~\ref{Sec:NumericalComparison} highlights the superior performance of R-SVRG through numerical comparisons with R-SGD on three problems.

The proposed R-SVRG is implemented in the MATLAB toolbox Manopt \cite{Boumal_Manopt_2014_s}.
The MATLAB code for all the proposed algorithms is available at \url{https://github.com/hiroyuki-kasai/RSOpt}.

\section{Riemannian optimization}
\label{Sec:GrassmannAndProblems}
Optimization on Riemannian manifolds, or {\it Riemannian optimization}, seeks a critical point of a given real-valued function, called the objective or cost function, defined on a smooth Riemannian manifold $\mathcal{M}$. One of the advantages of using Riemannian geometry tools is that the intrinsic properties of the manifold enable us to handle constrained optimization problems as unconstrained optimization problems.
This section introduces optimization on manifolds by summarizing \cite{Absil_OptAlgMatManifold_2008}.
Readers may refer to the various references therein as well as in \cite{Meyer_ICML_2011,Mishra_SIAMJOPT_2016} for further details.

Let $f\colon \mathcal{M} \to \mathbb{R}$ be a smooth real-valued function on manifold $\mathcal{M}$.
In optimization, we compute a minimum of $f$; typical methods for solving this minimization problem are {\it iterative algorithms} on manifold $\mathcal{M}$. In an iterative algorithm based on line search, with a given starting point $w_0 \in \mathcal{M}$, we generate a sequence $\{w_t\}_{t\geq0}$ on $\mathcal{M}$ that converges to $w^{*}$ whenever $w_0$ is in a neighborhood of $w^{*}$.
In an iterative optimization algorithm, we compute a search direction and then move in the search direction.
More specifically, an iteration on manifold $\mathcal{M}$ is performed by following geodesics emanating from $w_t$ and tangent to $\xi_{w_t}$ at $w_t$.
The notion of {\it geodesics} on Riemannian manifolds is a generalized concept of straight lines in Euclidean space. For any tangent vector $\xi \in T_w \mathcal{M}$ at $w \in \mathcal{M}$, there exists an interval $I$ about $0$ and a unique geodesic $\gamma_e(\cdot;w,\xi)\colon I \to \mathcal{M}$, such that $\gamma_e(0;w,\xi)=w$ and $\dot{\gamma_e}(0;w,\xi) = \xi$.
The {\it exponential mapping} ${\rm Exp}_{w}\colon T_{w}\mathcal{M} \to \mathcal{M}$ at $w \in \mathcal{M}$ is defined by geodesics emanating from $w$ as ${\rm Exp}_{w} \xi=\gamma_e(1;w,\xi)$ for $\xi \in T_w\mathcal{M}$.
If $\mathcal{M}$ is a complete manifold, the exponential mapping is defined for all vectors $\xi \in T_{w}\mathcal{M}$ \cite[Section~5.4]{Absil_OptAlgMatManifold_2008}.
We can thus obtain an update formula using the exponential mapping
\begin{equation*}
	w_{t+1}  =  {\rm Exp}_{w_t}\left(s_t \xi_{w_t}\right),	
\end{equation*}
where the search direction $\xi_{w_t}$ is in the tangent space $T_{w_t}\mathcal{M}$ of $\mathcal{M}$ at $w_t$, the scalar $s_t > 0$ is the step size, and ${\rm Exp}_{w_t}(\cdot)$ is the exponential mapping,
which induces a line search algorithm along the geodesics.
In addition, given two points $w$ and $z$ on $\mathcal{M}$, the \emph{logarithm mapping}, or simply {\it log mapping}, which is the inverse of the exponential mapping, maps $z$ to a vector $\xi \in T_{w}\mathcal{M}$ on the tangent space at $w$.
The log mapping satisfies ${\rm dist} (w,z) = \| {\rm Log}_{w} (z) \|_{w}$,
where ${\rm dist}\colon \mathcal{M} \times \mathcal{M} \to \mathbb{R}$ is the shortest distance between two points on $\mathcal{M}$ and $\|\cdot\|_w$ is the norm in $T_w\mathcal{M}$ defined through a Riemannian metric (see below).

The {\it steepest descent} or {\it full gradient descent} method for minimizing $f$ on $\mathcal{M}$ is an iterative algorithm obtained when $-{\rm grad}f(w_t)$ is used as the search direction $\xi_{w_t}$.
${\rm grad}f(w_t)$ is the {\it Riemannian gradient} of $f$ at $w_t$, which is computed according to the chosen metric $g$ at $w_t \in \mathcal{M}$.
Collecting each metric $g_w\colon T_w\mathcal{M} \times T_w\mathcal{M} \to \mathbb{R}$ over $w \in \mathcal{M}$ gives a family called a {\it Riemannian metric} on $\mathcal{M}$. $g_w(\xi_w, \zeta_w)$ is an inner product of elements $\xi_w$ and $\zeta_w$ in the tangent space $T_{w}\mathcal{M}$ at $w$.
Here, we use the notation $\langle\cdot, \cdot\rangle_w$ instead of $g(\cdot, \cdot)_w$ for simplicity.
The gradient ${\rm grad} f(w)$ is defined as the unique element of $T_{w}\mathcal{M}$ that satisfies
\begin{eqnarray*}
	 {\rm D} f(w)[\xi_w]& = & \left\langle{\rm grad} f(w), \xi_w\right\rangle_w, \ \ \forall \xi_w \in T_w \mathcal{M},
\end{eqnarray*}
where ${\rm D} f(w) \colon T_w \mathcal{M} \to \mathbb R$ is the {\it derivative} of $f$ at $w$.

In searching for the next point along a geodesic, we need to compute tangent vectors obtained by the exponential mapping, which are expensive to compute in general.
There are some Riemannian manifolds for which a closed form for the exponential mapping is not available.
Alternatively, we can use curves other than geodesics as long as a starting point and its tangent vector at the initial time are the same as those of the geodesics. A more general update formula is then written as
\begin{eqnarray*}
	w_{t+1} & = & R_{w_t}\left(s_t \xi_{w_t}\right),
\end{eqnarray*}
where $R_{w_t}$ is a {\it retraction}, which is any map $R_{w}\colon T_{w}\mathcal{M} \to \mathcal{M}$ that locally approximates the exponential mapping, up to the first order, on the manifold.
The definition of retraction is as follows \cite[Definition~4.1.1]{Absil_OptAlgMatManifold_2008}:\
\begin{definition}
$R\colon T\mathcal{M} \to \mathcal{M}$ is called a retraction on $\mathcal{M}$ if the restriction $R_w\colon T_w \mathcal{M} \to \mathcal{M}$ to $T_w \mathcal{M}$ for any $w \in \mathcal{M}$ satisfies all of the following:
\begin{enumerate}
\item
$R_w(0_w) = w$, where $0_w$ is the zero vector in $T_w \mathcal{M}$;
\item
${\rm D}R_w(0_w)[\xi] = \xi$ for any $\xi \in T_w \mathcal{M}$.
\end{enumerate}
\end{definition}
Retractions include the exponential mapping as a special case.
An advantage of using retractions is that the computational cost can be reduced compared to exponential mapping.
It is worth noting that the convergence properties for the exponential mapping usually hold for retractions as well.

In the R-SVRG proposed in Section \ref{Sec:R-SVRG},
we need to add tangent vectors that are in different tangent spaces, say $\tilde{w}$ and $w$ on $\mathcal{M}$.
A mathematically natural way to do so is to use the parallel translation operator.
Parallel translation $P_{\gamma}$ transports a vector field $\xi$ on the geodesic curve $\gamma$ that satisfies $P_{\gamma}^{a \leftarrow a} (\xi(a)) = \xi(a)$ and $\frac{{\rm D}}{{\rm d}t}(P_{\gamma}^{t \leftarrow a} \xi(a))=0$~\cite[Section 5.4]{Absil_OptAlgMatManifold_2008}, where $P_{\gamma}^{b \leftarrow a}$ is the parallel translation operator sending $\xi(a)$ to $\xi(b)$.
However, parallel translation is sometimes computationally expensive, and no explicit formula is available for some manifolds, such as the Stiefel manifold.
A {\it vector transport} on $\mathcal{M}$, which is a map $\mathcal{T}\colon T\mathcal{M} \oplus T\mathcal{M} \to T \mathcal{M}$, is used as an alternative.
The definition of vector transport is as follows \cite[Definition~8.1.1]{Absil_OptAlgMatManifold_2008}:
\begin{definition}
$\mathcal{T}\colon T\mathcal{M} \oplus T\mathcal{M} \to T\mathcal{M}$ is called a vector transport on $\mathcal{M}$ if it satisfies all of the following:
\begin{enumerate}
\item
$\mathcal{T}$ has an associated retraction $R$, i.e., for $w\in \mathcal{M}$ and $\xi, \eta \in T_w \mathcal{M}$, $\mathcal{T}_{\eta}(\xi)$ is a tangent vector at $R_w(\xi)$;
\item
$\mathcal{T}_{0_w}(\xi) = \xi$, where $\xi \in T_w \mathcal{M}$ and $w \in \mathcal{M}$;
\item
$\mathcal{T}_{\eta}(a\xi + b \zeta) = a\mathcal{T}_{\eta}(\xi) + b\mathcal{T}_{\eta}(\zeta)$, where $a, b \in \mathbb R$, $\eta, \xi, \zeta \in T_w \mathcal{M}$, and $w \in \mathcal{M}$.
\end{enumerate}
\end{definition}
With a vector transport $\mathcal{T}$, $\mathcal{T}_{\eta}(\xi)$ can be regarded as a transported vector of $\xi$ along $\eta$.
Parallel translation is an example of vector transport.
In the following, we use the notations $P_{\eta}$ and $P_{\gamma}^{z \leftarrow w}$ interchangeably, where $\gamma$ is a curve connecting $w$ and $z$ on $\mathcal{M}$ defined by retraction $R$ as $\gamma(\tau) := R_w(\tau\eta)$ with $z = R_w(\eta)$.
We also omit the subscript $\gamma$ when there is no confusion about the curve along which we transport a vector.

\section{Riemannian stochastic variance reduced gradient algorithm}
\label{Sec:R-SVRG}

After a brief explanation of the variance reduced gradient variants in Euclidean space, we describe the proposed Riemannian stochastic variance reduced gradient algorithm on Riemannian manifolds.

\subsection{Variance reduced gradient variants in Euclidean space}
The SGD update in Euclidean space is $w_{t+1}  =  w_{t}-\alpha v_t$, where $v_t$ is a randomly selected vector called the {\it stochastic gradient} and $\alpha$ is the step size. SGD assumes an {\it unbiased estimator} of the full gradient as $\mathbb{E}_n[\nabla f_n(w_t)] = \nabla f (w_t)$. Many recent variants of the variance reduced gradient of SGD attempt to reduce its variance $\mathbb{E}[\| v_t - \nabla f(w_t)\|^2]$ as $t$ increases to achieve better convergence \cite{Defazio_NIPS_2014_s,Johnson_NIPS_2013_s, Mairal_SIAMJOPT_2015,Roux_NIPS_2012_s,Shalev_arXiv_2012_s,Shalev_JMLR_2013_s,Zhang_SIAMJO_2014_s}. SVRG, proposed in \cite{Johnson_NIPS_2013_s}, introduces an explicit variance reduction strategy with double loops, where the $s$-th outer loop, called the $s$-th {\it epoch}, has $m_s$ inner iterations. SVRG first keeps $\tilde{w}^{s-1}=w_{m_{s-1}}^{s-1}$ or $\tilde{w}^{s-1}=w_{t}^{s-1}$ for randomly chosen $t\in\{1,2,\ldots,m_{s-1}\}$ at the end of the $(s\!\!-\!\!1)$-th epoch and also sets the initial value of the $s$-th epoch to $w_0^s=\tilde{w}^{s-1}$. It then computes a full gradient $\nabla f(\tilde{w}^{s-1})$. Subsequently, denoting the selected random index $i \in \{1, 2,\ldots, N\}$ by $i_t^s$, SVRG randomly picks the $i_t^s$-th sample for each $t \geq 1$ at $s \geq 1$ and computes the {\it modified stochastic gradient} $v_t^s$ as
\begin{eqnarray}
\label{Eq:E-SVRG}
v_t^s & = & \nabla f_{i_t^s} \left(w_{t-1}^{s}\right) -  \nabla f_{i_t^s} \left(\tilde{w}^{s-1}\right)  + \nabla f\left(\tilde{w}^{s-1}\right).
\end{eqnarray}
It should be noted that SVRG can be regarded as a special case of S2GD (semi-stochastic gradient descent), which differs in terms of the number of inner loop iterations chosen \cite{Konecny_arXiv_2013}.

\subsection{Proposed Riemannian extension of SVRG (R-SVRG)}
We propose a Riemannian extension of SVRG on a Riemannian manifold $\mathcal{M}$, called R-SVRG. Here, we denote the Riemannian stochastic gradient for the $i_t^s$-th sample as $\gradf_{i_t^s}$ and the {\it modified Riemannian stochastic gradient} as $\xi_t^s$ instead of $v_t^s$, in order to indicate the differences from the Euclidean case.

R-SVRG reduces the variance of the Riemannian stochastic gradient analogously to the SVRG algorithm in the Euclidean case. More specifically, R-SVRG keeps $\tilde{{w}}^{s-1} \in \mathcal{M}$ after $m_{s-1}$ stochastic update steps of the $(s\!\!-\!\!1)$-th epoch, and computes the full Riemannian gradient $\gradf (\tilde{{w}}^{s-1})=\frac{1}{N} \sum_{i=1}^{N} \gradf_i(\tilde{{w}}^{s-1})$ only for this stored $\tilde{{w}}^{s-1}$. The algorithm also computes the Riemannian stochastic gradient $\gradf_{i_t^s}(\tilde{{w}}^{s-1})$ that corresponds to the $i_t^s$-th sample.
Picking the $i_t^s$-th sample for each $t$-th inner iteration of the $s$-th epoch at ${w}_{t-1}^{s}$, we calculate $\xi_t^s$ in the same way as $v_t^s$ in (\ref{Eq:E-SVRG}), i.e., by modifying the stochastic gradient $\gradf_{i_t^s}({w}_{t-1}^{s})$ using both $\gradf (\tilde{{w}}^{s-1})$ and $\gradf_{i_t^s}(\tilde{{w}}^{s-1})$.
Translating the right-hand side of (\ref{Eq:E-SVRG}) to manifold $\mathcal{M}$ involves the sum of $\gradf_{i_t^s}({w}_{t-1}^{s})$, $-\gradf_{i_t^s}(\tilde{{w}}^{s-1})$, and $\gradf(\tilde{{w}}^{s-1})$, which belong to two separate tangent spaces $T_{\scriptsize {w}_{t-1}^{s}}\mathcal{M}$ and  $T_{\scriptsize \tilde{{w}}^{s-1}}\mathcal{M}$.
This operation requires particular attention on a manifold, and a vector transport enables us to handle multiple elements on two separate tangent spaces flexibly.
More specifically, $\gradf_{i_t^s}(\tilde{{w}}^{s-1})$ and $\gradf(\tilde{{w}}^{s-1})$ are first transported to $T_{\scriptsize {w}_{t-1}^{s}}\mathcal{M}$ at the current point, ${w}_{t-1}^{s}$; then, they can be added to $\gradf_{i_t^s}({w}_{t-1}^{s})$ on $T_{\scriptsize {w}_{t-1}^{s}}\mathcal{M}$. Consequently, the modified Riemannian stochastic gradient $\xi_t^s$ at the $t$-th inner iteration of the $s$-th epoch is set to
\begin{eqnarray}
\label{Eq:R-SVRG-Grad-paralleltrans}
\xi_t^s & = & \gradf_{i_t^s}\left({w}_{t-1}^{s}\right) - \mathcal{T}_{\tilde{\eta}_{t-1}^s} \left(\gradf_{i_t^s}\left(\tilde{{w}}^{s-1}\right) - \gradf\left(\tilde{{w}}^{s-1}\right)\right), 
\end{eqnarray}
where $\mathcal{T}$ is a vector transport operator on $\mathcal{M}$ and $\tilde{\eta}_{t-1}^s \in T_{\tilde{w}^{s-1}}\mathcal{M}$ satisfies $R_{\tilde{w}^{s-1}}(\tilde{\eta}_{t-1}^s) = w_{t-1}^s$.
Specifically, we need to calculate the tangent vector from $\tilde{{w}}^{s-1}$ to ${w}_{t-1}^{s}$,
which is given by the inverse of the retraction, i.e., $R^{-1}$, if available.
Consequently, the final update rule of R-SVRG is defined as ${w}_{t}^{s} =  R_{\scriptsize {w}_{t-1}^{s}}( - \alpha_{t-1}^s \xi_t^s)$, where $\alpha_{t-1}^s > 0$ is the step size at $w_{t-1}^s$.
As shown in our local convergence analysis (Section \ref{AppenSec:LocalConvergenceAnalysis}), $\alpha_{t-1}^s$ can be fixed once the iterate becomes sufficiently close to a solution.

Let $\mathbb{E}_{i_t^s}[\cdot]$ be the expectation with respect to the random choice of $i_t^s$, conditioned on all the randomness introduced up to the $t$-th iteration of the inner loop of the $s$-th epoch.
Conditioned on ${w}_{t-1}^s$, we take the expectation with respect to $i_t^s$ and obtain
\begin{eqnarray*}
\mathbb{E}_{i_t^s}\left[\xi_t^s\right] & = & 
\mathbb{E}_{i_t^s}\left[\gradf_{i_t^s}({w}_{t-1}^{s})\right] - 
\mathcal{T}_{\tilde{\eta}_{t-1}^s}
\left(
\mathbb{E}_{i_t^s}\left[\gradf_{i_t^s}\left(\tilde{{w}}^{s-1}\right) \right] -
\gradf\left(\tilde{{w}}^{s-1}\right)
\right) \nonumber\\
& = & \gradf\left({w}_{t-1}^{s}\right) -
\mathcal{T}_{\tilde{\eta}_{t-1}^s}
\left(
\gradf\left(\tilde{{w}}^{s-1}\right)  -
\gradf\left(\tilde{{w}}^{s-1}\right)
\right)  \nonumber\\
& = & \gradf\left({w}_{t-1}^{s}\right).
\end{eqnarray*}

The theoretical convergence analysis of the Euclidean SVRG algorithm assumes that the initial vector ${w}_0^s$ of the $s$-th epoch is set to the average (or a random) vector of the $(s\!\!-\!\!1)$-th epoch \cite[Figure~1]{Johnson_NIPS_2013_s}. However, the set of the last vectors in the $(s\!\!-\!\!1)$-th epoch, i.e., ${w}_{m_{s-1}}^{s-1}$, shows superior performance in the Euclidean SVRG algorithm. Therefore, for our local convergence rate analysis in Theorem \ref{Thm:LocalConvergence}, this study also uses, as {\bf option I}, the mean value of $\tilde{{w}}^{s} = g_{m_s}({w}_1^s, w_2^s,\dots, {w}_{m_s}^s)$ as $\tilde{{w}}^s$, where $g_n({w}_1, w_2,\ldots, {w}_n)$ is the Riemannian centroid on the manifold.
Alternatively, we can also simply choose $\tilde{{w}}^{s}={w}_t^s$ for $t \in \{1, 2,\ldots, m_s\}$ at random.
In addition, as {\bf option II}, we can use the last iterate in the $s$-th epoch, i.e., $\tilde{{w}}^{s}={w}_{m_s}^s$.
Note that {\bf option II} is always of practical use because no additional computation is needed to obtain $\tilde{w}^s$. However, if we can compute the Riemannian centroid of $w_1^s, w_2^s, \dots, w_{m_s}^s$ relatively cheaply, {\bf option I} may attain a better linear convergence rate.
In the global convergence analysis in Section \ref{Sec:GlobalAnalysis}, we use {\bf option II}, whereas both options are analyzed in the local convergence analysis in Section \ref{Sec:LocalAnalysis}.
The overall algorithm is summarized in Algorithm \ref{Alg:R-SVRG}. 

\begin{algorithm}
\caption{R-SVRG for Problem \eqref{Prob:general}.}
\label{Alg:R-SVRG}
\begin{algorithmic}[1]
\REQUIRE{Update frequency $m_s>0$ and sequence $\{\alpha_t^s\}$ of positive step sizes.}
\STATE{Initialize $\tilde{{w}}^0$.}
\FOR{$s=1,2, \ldots$} 
\STATE{Calculate the full Riemannian gradient $\gradf(\tilde{{w}}^{s-1})$}.
\STATE{Store ${w}_0^s = \tilde{{w}}^{s-1}$.}
	\FOR{$t=1,2, \ldots, m_s$} 
	\STATE{Choose $i_t^s \in \{1, 2,\ldots, N\}$ uniformly at random.}
	\STATE{Calculate the tangent vector $\tilde{\eta}_{t-1}^s$ from $\tilde{{w}}^{s-1}$ to ${w}_{t-1}^{s}$ satisfying $R_{\tilde{{w}}^{s-1}}(\tilde{\eta}_{t-1}^s) = {w}_{t-1}^s$}.
	\STATE{Calculate the modified Riemannian stochastic gradient $\xi_t^s$ by transporting $\gradf(\tilde{{w}}^{s-1})$ and $\gradf_{i_t^s}(\tilde{{w}}^{s-1})$ along $\tilde{\eta}_{t-1}^s$ as\\
	$\xi_t^s = \gradf_{i_t^s}({w}_{t-1}^{s}) -  \mathcal{T}_{\tilde{\eta}_{t-1}^s} \left(\gradf_{i_t^s}(\tilde{{w}}^{s-1})  -  \gradf(\tilde{{w}}^{s-1})\right)$.}
	\STATE{Update ${w}_t^s$ from ${w}_{t-1}^s$ as ${w}_t^s = R_{\scriptsize {w}_{t-1}^s}\left(- \alpha_{t-1}^s \xi_t^s \right)$ with retraction $R$.}
	\ENDFOR
	\STATE{{\bf option I}: $\tilde{{w}}^{s}=g_{m_s}({w}_1^s,w_2^s,\ldots,{w}_{m_s}^s)$.}	
	\STATE{{\bf option II}: $\tilde{{w}}^{s}={w}^s_{m_s}$.}
\ENDFOR
\end{algorithmic}
\end{algorithm}

In addition, variants of the variance reduced SGD initially require a full gradient calculation at every epoch.
This initially results in great overhead compared to the ordinary SGD algorithm and eventually induces a {\it cold-start} property in these variants.
To avoid this, the use of the standard SGD update has been proposed only for the first epoch in Euclidean space~\cite{Konecny_arXiv_2013}.
We also adopt this simple modification of R-SVRG, denoted as R-SVRG+; we do not analyze this extension but leave it as an open problem. 

As mentioned earlier, each iteration of R-SVRG has double loops to reduce the variance of the Riemannian stochastic gradient $\gradf_{i_t^s}(\tilde{{w}}^{s-1})$.
The $s$-th epoch, i.e., the outer loop, requires $N+2m_s$ gradient evaluations, where $N$ is for the full gradient $\gradf(\tilde{{w}}^{s-1})$ at the beginning of each $s$-th epoch and $2m_s$ is for the inner iterations, since each inner step needs two gradient evaluations, i.e., $\gradf_{i_t^s}({w}_{t-1}^{s})$ and $\gradf_{i_t^s}(\tilde{{w}}^{s-1})$. However, if $\gradf_{i_t^s}(\tilde{{w}}^{s-1})$ for each sample is stored at the beginning of the $s$-th epoch, as in SAG, the evaluations over all the inner loops result in $m_s$. Finally, the $s$-th epoch requires $N+m_s$ evaluations. It is natural to choose an $m_s$ of the same order as $N$ but slightly larger (e.g., $m_s = 5N$ for nonconvex problems has been suggested in \cite{Johnson_NIPS_2013_s}).

\section{Global convergence analysis}
\label{Sec:GlobalAnalysis}
In this section, we present a global convergence analysis of Algorithm \ref{Alg:R-SVRG} for Problem \eqref{Prob:general} after introducing some assumptions.
Throughout this section, we let $R$ and $\mathcal{T}$ denote a retraction and vector transport used in Algorithm \ref{Alg:R-SVRG}, respectively,
and we make the following assumptions.
\begin{assumption}
\label{assump:T_isometry}
The retraction $R$ is such that $R_w\colon T_w \mathcal{M} \to \mathcal{M}$ for any $w \in \mathcal{M}$ is of class $C^2$.
The vector transport $\mathcal{T}$ is continuous and isometric on $\mathcal{M}$, i.e., for any $w \in \mathcal{M}$ and $\xi, \eta \in T_w \mathcal{M}$, $\|\mathcal{T}_{\eta}(\xi)\|_{R_w(\eta)} = \|\xi\|_w$.
\end{assumption}
Note that we can construct an isometric vector transport as in \cite{huang2015riemannian,huang2015broyden} such that Assumption \ref{assump:T_isometry} holds.

\begin{assumption}
\label{Appen_assump:C2}
The objective function $f$ is thrice continuously differentiable and its components $f_1, f_2, \dots, f_N$ are twice continuously differentiable.
\end{assumption}

\begin{assumption}
\label{assump:GlobalTRN}
For a sequence $\{w_t^s\}$ generated by Algorithm \ref{Alg:R-SVRG}, there exists a compact and connected set $K \subset \mathcal{M}$ such that $w_t^s \in K$ for all $s, t \ge 0$.
In addition, for each $s\ge 1$ and $t \ge 0$, there exists $\tilde{\eta}_{t-1}^s \in T_{\tilde{w}^{s-1}}\mathcal{M}$ such that $R_{\tilde{w}^{s-1}}(\tilde{\eta}_{t-1}^s) = w_{t-1}^s$.
Furthermore, there exists $I>0$ such that for any $z \in K$, $R_z(\cdot)$ is defined in a ball $\mathbb{B}(0_z, I) \subset T_z \mathcal{M}$, which is centered at the origin $0_z$ in $T_z \mathcal{M}$ with radius $I$.
\end{assumption}
Note that the existence of $\tilde{\eta}_{t-1}^s$ in Assumption \ref{assump:GlobalTRN} is guaranteed if $R_z(\cdot)$ for any $z \in K$ is a diffeomorphism in a ball $\mathbb{B}(0_z, \rho)$ that satisfies $K \subset R_z(\mathbb{B}(0_z, \rho))$.
This assumption is a weakened version of the assumptions in some existing studies for special cases.
For example, in \cite{Zhang_NIPS_2016}, where the exponential mapping ${\rm Exp}$ is used as a retraction, the inverse of ${\rm Exp}$ is assumed to exist at all points $\{w_t^s\}$ in their global convergence analysis.
Here, we note that if $I$ in Assumption \ref{assump:GlobalTRN} is too small, then $R_{\tilde{w}^{s-1}}^{-1}$ may not be defined in a neighborhood of $w_{t-1}^s$.
We just suppose that there exists $\tilde{\eta}_{t-1}^s$ that is mapped to $w_{t-1}^s$ by $R_{\tilde{w}^{s-1}}$.
Further analysis with respect to a specific manifold in question may be required to specify the value of $I$.

In the global convergence analysis, we also assume that the sequence of step sizes $\{\alpha_t^s\}_{t \geq 0, s \geq 1}$ satisfies the usual condition in stochastic approximation as follows:
\begin{assumption}
\label{assump:stepsizes}
The sequence $\{\alpha_t^s\}$ of step sizes satisfies
\begin{eqnarray}
\label{Eq:StepsizeCondition}
\sum \left(\alpha_t^s\right)^2 \ <\  \infty {\ \ \ \rm and\ \ \ } \sum \alpha_t^s \ =\  \infty,
\end{eqnarray}
where $\sum$ denotes $\sum_{s = 1}^{\infty}\sum_{t=0}^{m_s}$.
\end{assumption}
This condition is satisfied, e.g., if $\{m_s\}$ is upper bounded and $\alpha_t^s = \alpha_0(1+\alpha_0 \lambda \lfloor k / m_s \rfloor)^{-1}$ with positive constants $\alpha_0$ and $\lambda$, where $k$ is the total iteration number depending on $s$ and $t$ and $\lfloor\cdot\rfloor$ denotes the floor function.
We also note the following proposition introduced in \cite{Fisk_1965_s}:
\begin{proposition}[\cite{Fisk_1965_s}]
\label{Prop:FiskPro}
Let $(X_n)_{n \in \mathbb{N}}$ be a nonnegative stochastic process with bounded positive variations, i.e., such that $\sum_{n=0}^{\infty} \mathbb{E}[\mathbb{E}[X_{n+1} - X_n \mid \mathcal{F}_n]^{+}] < \infty$, where $X^+$ denotes the quantity $\max\{X,0\}$ for a random variable $X$ and $\mathcal{F}_n$ is the increasing sequence of $\sigma$-algebras generated by the variables just before time $n$.
Then, the process is a quasi-martingale, i.e., 
\begin{eqnarray*}
\sum_{n=0}^{\infty} \left|\mathbb{E}\left[ X_{n+1} - X_n \mid \mathcal{F}_n\right] \right| < \infty \ \ \ {\rm a.s.},\ {\rm and\ } X_n\ {\rm converges\ } {\rm a.s}. 
\end{eqnarray*}
\end{proposition}

Now, we give the a.s. convergence of the proposed algorithm under the assumption that the generated sequence is in a compact set.

\begin{theorem}
\label{Thm:GlobalConvergence}
Suppose Assumptions \ref{assump:T_isometry}--\ref{assump:GlobalTRN} and consider Algorithm \ref{Alg:R-SVRG} with {\bf option II} and step sizes $\{\alpha_t^s\}$ satisfying Assumption \ref{assump:stepsizes} on a Riemannian manifold $\mathcal{M}$.
If $f\ge 0$, then $\{f(w_t^s)\}$ converges a.s. and $\gradf (w_t^s) \to 0$ a.s.
\end{theorem}

\begin{proof}
The claim is proved similarly to the proof of the standard Riemannian SGD (see \cite{Bonnabel_IEEETAC_2013_s}).
Since $K$ is compact, all continuous functions on $K$ can be bounded.
Therefore, there exists $C > 0$ such that for all $w \in K$ and $n \in \{1,2,\ldots,N\}$, we have $\| \gradf (w) \|_{w} \leq C$ and $\| \gradf_{n}(w)\|_{w} \le C$.
We use $C':= 3C$ in the following.
Moreover, as $\alpha_t^s \to 0$, there exists $s_0$ such that for $s \geq s_0$, we have $\alpha_t^s < 1$ and $\alpha_t^s C' < I$.
Now, suppose that $s \geq s_0$.
Let $\tilde{\eta}_{t}^s$ satisfy $R_{\tilde{w}^{s-1}}(\tilde{\eta}_{t}^s) = w_{t}^s$.
The existence of such $\tilde{\eta}_{t}^s$ is guaranteed from Assumption \ref{assump:GlobalTRN}.
It follows from the triangle inequality that
\begin{eqnarray*}
\|\xi_{t+1}^s \|_{w_{t}^s} & = & \left\| \gradf_{i_{t+1}^s}\left(w_{t}^{s}\right) -  \mathcal{T}_{\tilde{\eta}_{t}^s} \left(\gradf_{i_{t+1}^s}\left(\tilde{w}^{s-1}\right) \right) + \mathcal{T}_{\tilde{\eta}_{t}^s} \left(\gradf\left(\tilde{w}^{s-1}\right)\right) \right\|_{w_{t}^s} \\
& \le & C + C + C = C',
\end{eqnarray*}
where we have used the assumption that $\mathcal{T}$ is an isometry.
Therefore, we have
\begin{equation*}
\left\| -\alpha_t^s \xi_{t+1}^s \right\|_{w_t^s} \le \alpha_t^s C' < I.
\end{equation*}
Hence, it follows from Assumption \ref{assump:GlobalTRN} that there exists a curve $R_{w_t^s}(-\tau \alpha_t^s \xi_{t+1}^s)_{0 \leq \tau \leq 1}$ linking $w_t^s$ and $w_{t+1}^s$.
Defining $g(\tau;w,\xi) := (f\circ R_{w})(-\tau \xi)$, there exists a constant $k_1$ such that
\begin{equation*}
\frac{d^2}{d\tau^2}g(\tau;w,\xi) \le 2k_1, \qquad \tau \in  [0,1],\ \xi \in \mathbb{B}(0_w, I), \ w \in K,
\end{equation*}
since $[0, 1] \times \{(w, \xi) \mid \xi \in \mathbb{B}(0_w, I), w \in K\}$ is compact.
A trivial relation
\begin{equation*}
g(\alpha;w,\xi) = g(0;w,\xi) + g'(0;w,\xi) \alpha + \alpha^2 \int_0^1 (1-\tau) g''(\alpha\tau;w,\xi) d\tau
\end{equation*}
together with $\alpha_t^s < 1$ then implies that
\begin{align*}
f\left(w_{t+1}^s\right) - f\left(w_{t}^s\right) =&f\left(R_{w_t^s}\left(-\alpha_t^s \xi_{t+1}^s\right)\right) - f\left(w_t^s\right) \\
= & g\left(\alpha_t^s;w_t^s,\xi_{t+1}^s\right) - g\left(0;w_t^s,\xi_{t+1}^s\right)\\
= & - \alpha_t^s 
\left\langle \xi_{t+1}^s, \gradf \left(w_t^s\right) 
\right\rangle_{w_t^s} + \left(\alpha_t^s\right)^2 \int_0^1 (1-\tau) g''\left(\alpha_t^s\tau\right)d\tau\\
\leq & - \alpha_t^s 
\left\langle \xi_{t+1}^s, \gradf \left(w_t^s\right) 
\right\rangle_{w_t^s} + \left(\alpha_t^s\right)^2 k_1.
\end{align*}
Let $\mathcal{F}_t^s$ be the increasing sequence of $\sigma$-algebras defined by
\begin{eqnarray*}
\mathcal{F}_t^s=\left\{i_1^1, i_2^1,\ldots, i_{m_1}^1,i_1^2,i_2^2,\dots,i_{m_2}^2, \ldots, i_1^{s-1},i_2^{s-1}, \ldots, i_{m_{s-1}}^{s-1}, i_1^{s}, i_2^s, \ldots, i_{t-1}^{s}\right\}.
\end{eqnarray*}
Since $w_t^s$ is computed from $i_1^1, i_2^1,\ldots, i_{t}^s$, it is measurable in $\mathcal{F}_{t+1}^s$. As $i_{t+1}^s$ is independent of $\mathcal{F}_{t+1}^s$, we have
\begin{align*}
&\mathbb{E}\left[\xi_{t+1}^s \mid \mathcal{F}_{t+1}^s\right]\\
= &\mathbb{E}_{i_{t+1}^s}\left[\xi_{t+1}^s\right]\\
= & \mathbb{E}_{i_{t+1}^s}\left[\gradf_{i_{t+1}^s}\left(w_t^s\right)-\mathcal{T}_{\tilde{\eta}_{t}^s}\left(\gradf_{i_{t+1}^s}\left(\tilde{w}^{s-1}\right)\right)+ \mathcal{T}_{\tilde{\eta}_{t}^s}\left(\gradf \left(\tilde{w}^{s-1}\right)\right)\right]\\
= &\gradf \left(w_t^s\right)
\end{align*}
from the linearity of $\mathcal{T}_{\tilde{\eta}_t^s}$. Therefore, it holds that
\begin{eqnarray*}
& & \mathbb{E}\left[\left\langle \xi_{t+1}^s, \gradf \left(w_t^s\right)\right\rangle_{w_t^s} \mid \mathcal{F}_{t+1}^s\right]
= \left\| \gradf \left(w_t^s\right)\right\|_{w_t^s}^2,
\end{eqnarray*}
which yields
\begin{eqnarray}
\label{Eq:ExpectationDiff}
\mathbb{E}\left[f\left(w_{t+1}^s\right) - f\left(w_{t}^s\right) \mid \mathcal{F}_{t+1}^s\right] & \leq & -\alpha_t^s \left\| \gradf \left(w_t^s\right)\right\|_{w_t^s}^2 + \left(\alpha_t^s\right)^2 k_1 \le \left(\alpha_t^s\right)^2 k_1.
\end{eqnarray}

We reindex the sequence $\{w_0^1, w_1^1,\dots,w_{m_1}^1(=w_0^2),w_1^2,w_2^2,\dots,w_{m_2}^2,\dots,w_t^s,\dots\}$\\
as
$\{w_1,w_2, \dots, w_k, \dots\}$.
We also similarly reindex $\{\alpha_t^s\}$.
As $f(w_{k}) \geq 0$, \eqref{Eq:ExpectationDiff} proves that $\{f(w_{k}) +\sum_{l=k}^{\infty} (\alpha_{l})^2 k_1\}$ is a nonnegative supermartingale. Hence, it converges a.s.
Returning to the original indexing, this implies that $\{f(w_{t}^s)\}$ converges a.s. Moreover, summing the inequalities, we have
\begin{eqnarray}
\label{Eq:ExpectationDiff2}
\sum_{s \geq s_0,t} \alpha_t^s \left\| \gradf \left(w_t^s\right) \right\|_{w_t^s}^2 & \leq & - \sum_{s \geq s_0,t} \mathbb{E}\left[f\left(w_{t+1}^s\right) - f\left(w_{t}^s\right) \mid \mathcal{F}_{t+1}^s\right] + \sum_{s \geq s_0,t}  \left(\alpha_t^s\right)^2 k_1.
\end{eqnarray}
Here, we prove that the right-hand side is bounded.
By doing so, we can conclude the convergence of the left term.

Summing \eqref{Eq:ExpectationDiff} over $s$ and $t$, we have
\begin{equation*}
\sum_{s\ge s_0, t}\mathbb{E}\left[\mathbb{E}\left[f\left(w_{t+1}^s\right) - f\left(w_t^s\right) \mid \mathcal{F}_{t+1}^s\right]^+\right] < \infty
\end{equation*}
by Assumption \eqref{Eq:StepsizeCondition}.
It then follows from Proposition \ref{Prop:FiskPro} that
\begin{equation*}
\left| - \sum_{s \geq s_0,t} \mathbb{E}\left[f\left(w_{t+1}^s\right) - f\left(w_{t}^s\right) \mid \mathcal{F}_{t+1}^s\right]\right|
\leq \sum_{s \geq s_0,t} \left|\mathbb{E}\left[f\left(w_{t+1}^s\right) - f\left(w_{t}^s\right) \mid \mathcal{F}_{t+1}^s\right]\right| < \infty,
\end{equation*}
which together with \eqref{Eq:ExpectationDiff2} implies that
$\sum_{s \ge s_0, t} \alpha_t^s \| \gradf (w_t^s) \|_{w_t^s}^2$ converges a.s.
If we further prove that $\{\|  \gradf (w_t^s) \|_{w_t^s}\}$ converges a.s., it can only converge to 0 a.s. because of Eq.~\eqref{Eq:StepsizeCondition}.

As in the discussion in \cite{Bonnabel_IEEETAC_2013_s}, we consider the nonnegative process $p_t^s = \|  \gradf (w_t^s) \|_{w_t^s}^2$.
Bounding the largest eigenvalue of the Hessian of $\|  \gradf(w) \|_{w}^2$ from above by $k_2$ on the compact set $K$ along the curve defined by the retraction $R$ linking $w_t^s$ and $w_{t+1}^s$, a Taylor expansion yields 
\begin{eqnarray*}
p_{t+1}^s - p_t^s \leq -2 \alpha_t^s \left\langle \gradf \left(w_t^s\right), {\rm Hess} f\left(w_t^s\right) \left[\xi_{t+1}^s\right] \right\rangle_{w_t^s} + \left(\alpha_t^s\right)^2 \left\|  \xi_{t+1}^s \right\|_{w_t^s}^2 k_2.
\end{eqnarray*}
Let $k_3$ be a constant such that $-k_3$ is a lower bound of the minimum eigenvalue of the Hessian of $f$ on $K$.
Then, we have
\begin{eqnarray*}
\mathbb{E}\left[p_{t+1}^s - p_t^s \mid \mathcal{F}_{t+1}^s\right] \leq 2 \alpha_t^s \left\|  \gradf \left(w_t^s\right) \right\|_{w_t^s}^2 k_3 + \left(\alpha_t^s\right)^2 C'^2 k_2.
\end{eqnarray*}
We have proved the fact that the infinite series of the terms of the right-hand side is finite,
implying that $\{p_t^s\}$ is a quasi-martingale, which further implies the a.s. convergence of $\{p_t^s\}$ to $0$, as claimed.
\end{proof}

Theorem \ref{Thm:GlobalConvergence} includes a global convergence of the algorithm with exponential mapping and parallel translation as a special case.
In this case, a sufficient condition for the assumptions can be easily written by the notions of injectivity radius as the following corollary.

\begin{corollary}
\label{Thm:GlobalConvergenceExp}
Suppose Assumption \ref{assump:T_isometry} and consider Algorithm \ref{Alg:R-SVRG} with {\bf option~II} and step sizes $\{\alpha_t^s\}$ satisfying Assumption \ref{assump:stepsizes} on a Riemannian manifold $\mathcal{M}$, where exponential mapping and parallel translation are used as retraction and vector transport, respectively.
Assume that $\mathcal{M}$ is connected and has an injectivity radius uniformly bounded from below by a positive constant.
Assume also that there exists a compact and connected set $K \subset \mathcal{M}$ such that $w_t^s \in K$ for all $s, t \ge 0$.
If $f\ge 0$, then $\{f(w_t^s)\}$ converges a.s. and $\gradf (w_t^s) \to 0$ a.s.
\end{corollary}

\section{Local convergence rate analysis}
\label{Sec:LocalAnalysis}
In this section, we show the local convergence rate analysis of the R-SVRG algorithm;
we analyze the convergence of any sequences generated by Algorithm \ref{Alg:R-SVRG} that are contained in a sufficiently small neighborhood of a local minimum point of the objective function.
Hence, we can assume that the objective function is strongly convex in such a neighborhood.
We first give formal expressions of these assumptions and then analyze the local convergence rate of the algorithm.

\subsection{Assumptions and existing lemmas}
We again suppose Assumption \ref{assump:T_isometry} and make the following assumption, which is a weaker version of Assumption \ref{Appen_assump:C2}.
\begin{assumption}
\label{assump:f_C2}
The objective function $f$ and its components $f_1, f_2, \dots, f_N$ are twice continuously differentiable.
\end{assumption}

Let $w^*$ be a critical point of $f$.
As discussed in \cite{huang2015broyden}, for a positive constant $\rho$, a $\rho$-totally retractive neighborhood $\Omega$ of $w \in \mathcal{M}$ is a neighborhood such that for all $z \in \Omega$, $\Omega \subset R_z(\mathbb B(0_z,\rho))$, and $R_z(\cdot)$ is a diffeomorphism on $\mathbb B(0_z,\rho)$, which is the ball in $T_{z} \mathcal{M}$ with center $0_z$ and radius $\rho$, where $0_z$ is the zero vector in $T_z\mathcal{M}$.
The concept of a totally retractive neighborhood is analogous to that of a totally normal neighborhood for exponential mapping.
Note that for each point, a $\rho$-totally retractive neighborhood for sufficiently small $\rho > 0$ is proved to exist~\cite{huang2015broyden}.
On the other hand, it should be noted that a concrete bound of the radii of locally retractive neighborhoods is difficult to specify.
In the local convergence analysis, we assume the following.
\begin{assumption}
\label{assump:Omega}
The sequence $\{w_t^s\}$ generated by Algorithm \ref{Alg:R-SVRG} continuously remains in totally retractive neighborhood $\Omega$ of critical point $w^*$ of $f$, i.e., $R_{w_t^s}(-\alpha\xi_{t+1}^s) \in \Omega$ for all $s, t \ge 0$ and for all $\alpha\in [0,\alpha_t^s]$.
In addition, the radius of $\Omega$ is sufficiently small.
\end{assumption}

We note that the assumption on the radius of $\Omega$ in Assumption \ref{assump:Omega} is for Assumption \ref{Assump:ret_convex} and for Lemma \ref{Lem_xi_norm}.

\begin{assumption}
\label{assump:T_Lipschitz}
There exists a constant $c_0$ such that vector transport $\mathcal{T}$ satisfies the following conditions:
\begin{equation*}
\left\| \mathcal{T}_\eta-\mathcal{T}_{R_\eta}\right\| \le c_0 \| \eta\|, \qquad \left\| \mathcal{T}^{-1}_\eta-\mathcal{T}^{-1}_{R_\eta}\right\| \le c_0 \| \eta\|,
\end{equation*}
where $\|\cdot\|$ denotes the induced (operator) norm of the Riemannian metric and $\mathcal{T}_R$ denotes the differentiated retraction, i.e.,
\begin{equation*}
\mathcal{T}_{R_{\eta}}(\xi) = {\rm D} R_w(\eta) [\xi], \qquad \eta, \xi \in T_w \mathcal{M}, \quad w \in \mathcal{M}.
\end{equation*}
\end{assumption}
Assumption \ref{assump:T_Lipschitz} states that the difference between $\mathcal{T}$ and $\mathcal{T}_R$ is small if the tangent vector along which a vector is transported is close to $0$. See \cite{huang2015broyden} for further details.

Furthermore, we assume the following.
\begin{assumption}
\label{Assump:ret_convex}
$f$ is strongly retraction-convex with respect to $R$ in $\Omega$; i.e., there exist two constants $0<a_0<a_1$ such that $a_0 \le \frac{d^2}{d\alpha^2} f(R_w(\alpha\eta)) \le a_1$ for all $w \in \Omega$, $\eta \in T_w \mathcal{M}$ with $\| \eta \|_{w} = 1$, and for all $\alpha$ satisfying $R_w(\tau \eta) \in \Omega$ for all $\tau \in [0,\alpha]$.
Furthermore, $f$ is strongly geodesically convex in $\Omega$, i.e., $f$ is strongly retraction-convex with respect to ${\rm Exp}$.
\end{assumption}
Letting $\Omega$ be smaller if necessary, we can guarantee the last assumption from the other assumptions using Lemma 3.1 in \cite{huang2015broyden}.

The next assumption states how close the modified stochastic gradient by retraction $R$ is to that obtained by the exponential mapping.
\begin{assumption}
\label{Assump:ret_exp}
Let $\{\xi_t^s\}$ be generated by Algorithm \ref{Alg:R-SVRG} with a fixed step size of $\alpha_t^s := \alpha$.
${\rm Exp}_w^{-1}$ for any $w\in \Omega$ is defined in $\Omega$ and
there exists constant $c_R > 0$ such that $\|{\rm Exp}_{w_{t-1}^s}^{-1}(w_t^s) - (-\alpha \xi_t^s)\|_{w_{t-1}^s} \le c_R \|\alpha \xi_t^s\|_{w_{t-1}^s}^2$.
\end{assumption}
If $R_{w_{t-1}^s}^{-1}$ is defined in a ball whose image by $R_{w_{t-1}^s}$ contains $\Omega$, then the above assumption translates into $\|{\rm Exp}_{w_{t-1}^s}^{-1}(w_t^s) - R_{w_{t-1}^s}^{-1}(w_t^s)\|_{w_{t-1}^s} \le c_R \|R_{w_{t-1}^s}^{-1}(w_t^s)\|_{w_{t-1}^s}^2$, which is natural when we assume sufficient smoothness of ${\rm Exp}$ and $R$ because $R$ coincides with ${\rm Exp}$ up to the first order.

In the remainder of this section, we introduce some existing lemmas to evaluate the differences between using retraction and vector transport and using exponential mapping and parallel translation, and the effects of the curvature of the manifold in question.
\begin{lemma}[In the proof of Lemma 3.9 in \cite{huang2015broyden}] \label{Lem:pseudo_Lipschitz}
Under Assumptions \ref{assump:f_C2} and \ref{assump:Omega}, there exists a constant $\beta>0$ such that
\begin{eqnarray}
     \label{pseudo_Lipschitz}
	\left\| P_{\gamma}^{w \leftarrow z} ({\rm grad} f(z)) - {\rm grad} f(w) \right\|_{w} & \leq & \beta {\rm dist}(z,w),
\end{eqnarray}
where $w$ and $z$ are in $\Omega$ and $\gamma$ is a curve $\gamma(\tau):=R_z(\tau\eta)$ for an arbitrary $\eta \in T_z \mathcal{M}$ defined by retraction $R$ on $\mathcal{M}$.
$P_{\gamma}^{w \leftarrow z}(\cdot)$ is a parallel translation operator along curve $\gamma$ from $z$ to $w$.
\end{lemma}
Note that curve $\gamma$ in this lemma is not necessarily the geodesic on $\mathcal{M}$.
Relation \eqref{pseudo_Lipschitz} is a generalization of the Lipschitz continuity condition.
\begin{lemma}[Lemma 3.5 in \cite{huang2015broyden}] \label{Lem:T_Lipschitz}
Let $\mathcal{T}\in C^{0}$ be a vector transport associated with the same retraction $R$ as that of the parallel transport $P\in C^{\infty}$.
Under Assumption \ref{assump:T_Lipschitz}, for any $\bar{w}\in\mathcal{M}$, there exists a constant $\theta>0$ and neighborhood $\mathcal{U}$ of $\bar{w}$ such that for all $w, z \in \mathcal{U}$,
\begin{equation*}
\left\|\mathcal{T}_\eta(\xi)-P_{\eta}(\xi)\right\|_{z} \le \theta\| \xi\|_{w} \| \eta\|_{w},
\end{equation*}
where $\xi, \eta \in T_w\mathcal{M}$ and $R_w(\eta)=z$.
\end{lemma}
We can derive the following lemma from the Taylor expansion as in the proof of Lemma 3.2 in \cite{huang2015broyden}.
\begin{lemma}[In the proof of Lemma 3.2 in \cite{huang2015broyden}]
\label{Lem:mean-value}
Under Assumptions \ref{assump:Omega}--\ref{Assump:ret_convex}, there exists a positive real number $\sigma$ such that
\begin{eqnarray}
\label{Eq:mean-value}
f(z) \ge f(w) + \left\langle {\rm Exp}_{w}^{-1}( z), \gradf(w)\right\rangle_{w} + \frac{\sigma}{2} \left\|{\rm Exp}_{w}^{-1}(z)\right\|_{w}^2, \qquad w, z \in \Omega.
\end{eqnarray}
\end{lemma} 
\begin{proof}
From the assumptions, there exists $\sigma>0$ such that $\frac{d^2}{d\alpha^2}f({\rm Exp}_w(\alpha\eta)) \ge \sigma$ for all $w \in \Omega$, $\eta \in T_w \mathcal{M}$ with $\|\eta\|_w=1$, and for all $\alpha$ such that ${\rm Exp}_w(\tau\eta) \in \Omega$ for all $\tau \in [0, \alpha]$.
From Taylor's theorem, we can conclude that this $\sigma$ satisfies the claim.
\end{proof}
If we replace ${\rm Exp}$ in Lemma \ref{Lem:mean-value} with retraction $R$ on $\mathcal{M}$, we can obtain a similar result for $R$,
which is shown in Lemma 3.2 in \cite{huang2015broyden} where the constant $a_0$ corresponds to $\sigma$ in our Lemma \ref{Lem:mean-value}.
However, Lemma \ref{Lem:mean-value} for ${\rm Exp}$ is sufficient in the following discussion.

From Lemma 3 in \cite{huang2015riemannian}, we have the following:
\begin{lemma}[Lemma 3 in \cite{huang2015riemannian}]
\label{lemma:retraction_dist}
Let $\mathcal{M}$ be a Riemannian manifold endowed with retraction $R$ and let $\bar{w} \in \mathcal{M}$.
Then, there exist $\mu > 0$, $\nu > 0$, and $\delta_{\mu,\nu}>0$ such that for all $w$ in a sufficiently small neighborhood of $\bar{w}$ and all $\xi \in T_w \mathcal{M}$ with $\|\xi\|_w \le \delta_{\mu,\nu}$, the inequalities
\begin{equation}
\label{Eq:tau}
\|\xi\|_w \le \mu {\rm dist} \left(w, R_w(\xi)\right) \quad {\rm and} \quad {\rm dist} \left(w, R_w(\xi)\right) \le \nu \|\xi\|_w
\end{equation}
hold.
\end{lemma}
Since we have ${\rm dist}(w, {\rm Exp}_w(\xi)) = \| \xi \|_w$, \eqref{Eq:tau} is equivalent to ${\rm dist}(w, {\rm Exp}_w(\xi)) \le \mu {\rm dist} (w, R_w(\xi))$ and ${\rm dist} (w, R_w(\xi)) \le \nu{\rm dist}(w, {\rm Exp}_w(\xi))$, which give relations between the exponential mapping and a general retraction.

Now, we introduce Lemma 6 in \cite{Zhang_JMLR_2016_s} to evaluate the distance between $w_{t}^s$ and $w^{*}$ using the smoothness of our objective function.
In the following, an Alexandrov space is defined as a length space whose curvature is bounded.
\begin{lemma}[Lemma 6 in \cite{Zhang_JMLR_2016_s}]
\label{LemZhang}
If $a$, $b$, and $c$ are the side lengths of a geodesic triangle in an Alexandrov space with curvature lower-bounded by $\kappa$, and $A$ is the angle between sides $b$ and $c$, then
\begin{equation*}
a^2 \le \frac{\sqrt{|\kappa|}c}{\tanh\left(\sqrt{|\kappa|}c\right)}b^2 + c^2 - 2bc\cos(A).
\end{equation*}
\end{lemma}

\subsection{Local convergence rate analysis with retraction and vector transport}
\label{AppenSec:LocalConvergenceAnalysis}

We now demonstrate the local convergence properties of the R-SVRG algorithm (i.e., local convergence to local minimizers) and its convergence rate. 
The main theorem (Theorem \ref{Thm:LocalConvergence}) follows three lemmas, the proofs of which are provided in Appendix~\ref{Appendix}.

The following lemma shows a property of the Riemannian centroid on a general Riemannian manifold.
\begin{lemma}
\label{AppenLem:KarcherMeanDistance}
Let $w_1, w_2,\dots,w_m$ be points on a Riemannian manifold $\mathcal{M}$ and let $w$ be the Riemannian centroid of the $m$ points.
For an arbitrary point $p$ on $\mathcal{M}$, we have
\begin{eqnarray*}
({\rm dist}(p,w))^2\le \frac{4}{m}\sum_{i=1}^m\left({\rm dist}\left(p,w_i\right)\right)^2.
\end{eqnarray*}
\end{lemma}

Then, we state that the norms of modified stochastic gradients in R-SVRG are sufficiently small under Assumptions \ref{assump:T_isometry} and \ref{assump:f_C2}--\ref{assump:T_Lipschitz}.
\begin{lemma}\label{Lem_xi_norm}
Under Assumptions \ref{assump:T_isometry} and \ref{assump:f_C2}--\ref{assump:T_Lipschitz}, the norm of $\xi_t^s$ computed by \eqref{Eq:R-SVRG-Grad-paralleltrans} is sufficiently small; i.e.,
for any $\varepsilon >0$, there exists $r_0>0$ such that $\|\xi_t^s\|_{w_{t-1}^s} \le \varepsilon$ for $r$ satisfying $r < r_0 $, where $r$ is the radius of $\Omega$.
\end{lemma}

We now provide the upper bound of the variance of $\xi_t^s$ as follows.

\begin{lemma}
\label{AppenLem:UpperBoundVariance}
Suppose Assumptions \ref{assump:T_isometry} and \ref{assump:f_C2}--\ref{assump:T_Lipschitz}, which guarantee Lemmas \ref{Lem:pseudo_Lipschitz}, \ref{Lem:T_Lipschitz}, and \ref{lemma:retraction_dist} for $\bar{w} = w^*$.
Let $\beta>0$ be a constant such that
\begin{eqnarray*}
	\left\| P_{\gamma}^{w \leftarrow z} \left({\rm grad} f_n(z)\right) - {\rm grad} f_n(w) \right\|_{w} & \leq & \beta {\rm dist}(z,w),\qquad w, z \in \Omega,\ n = 1,2,\dots, N.
\end{eqnarray*}
The existence of such a $\beta$ is guaranteed by Lemma \ref{Lem:pseudo_Lipschitz}.
The upper bound of the variance of $\xi_t^s$ is given by
\begin{equation}
\label{Append_Eq:UpperBoundVariance}
	\mathbb{E}_{i_t^s}\left[\left\| \xi_t^s \right\|_{w_{t-1}^s}^2\right] \leq
4\left(\beta^2+\mu^2C^2\theta^2\right)\left(7\left({\rm dist}\left(w_{t-1}^s,w^*\right)\right)^2 + 4\left({\rm dist}\left(\tilde{w}^{s-1},w^*\right)\right)^2\right),
\end{equation}
where the constant $\theta$ corresponds to that in Lemma \ref{Lem:T_Lipschitz},
$C$ is the upper bound of $\|\gradf_n(w)\|,\ n=1, 2, \dots, N$ for $w \in \Omega$,
and $\mu > 0$ appears in \eqref{Eq:tau}.
\end{lemma}

We proceed to the main theorem and prove it using Lemmas \ref{AppenLem:KarcherMeanDistance}--\ref{AppenLem:UpperBoundVariance}.
\begin{theorem}
\label{Thm:LocalConvergence}
Let $\mathcal{M}$ be a Riemannian manifold whose curvature is lower-bounded by $\kappa$, $w^* \in \mathcal{M}$ be a nondegenerate local minimizer of $f$ (i.e., ${\rm grad} f(w^*)=0$ and the Hessian ${\rm Hess}f(w^*)$ of $f$ at $w^*$ is positive definite), $D$ be the diameter of the compact set $\Omega$,
and $\zeta:=\sqrt{|\kappa|}D/\tanh(\sqrt{|\kappa|}D)$ if $\kappa < 0$ and $\zeta:=1$ if $\kappa \geq 0$.
Suppose Assumptions \ref{assump:T_isometry} and \ref{assump:f_C2}--\ref{Assump:ret_exp}.
Let $c_R$ be a constant in Assumption \ref{Assump:ret_exp}, let $\beta$, $\mu$, $\theta$, and $C$ be the same as in Lemma \ref{AppenLem:UpperBoundVariance}, and let $\nu$ be a constant in Lemma \ref{lemma:retraction_dist}.
Suppose that $\sigma>0$ is a constant in Lemma \ref{Lem:mean-value} satisfying \eqref{Eq:mean-value}.
Let $\alpha$ be a positive number satisfying $0 < \alpha(\sigma-28(\zeta\nu+2c_R D)(\beta^2+\mu^2C^2\theta^2)\alpha) < 1$.
It then follows that for any sequence $\{\tilde{w}^s\}$ generated by Algorithm \ref{Alg:R-SVRG} with a fixed step size $\alpha_t^s:=\alpha$ and $m_s:=m$ converging to $w^*$, there exists $K>0$ such that for all $s>K$,
\begin{equation}
\label{Eq:LinearConvergence}
\mathbb{E}\left[\left({\rm dist}\left(\tilde{w}^s,w^*\right)\right)^2\right] \leq
\delta
\mathbb{E}\left[\left({\rm dist}\left(\tilde{w}^{s-1},w^*\right)\right)^2\right],
\end{equation}
where
\begin{equation*}
\delta :=
\begin{cases}
\displaystyle\frac{4\left(1+ 16m\left(\zeta\nu+2c_R D\right)\left(\beta^2+\mu^2C^2\theta^2\right)\alpha^2\right)}{ m\alpha\left(\sigma-28\left(\zeta\nu+2c_R D\right)\left(\beta^2+\mu^2C^2\theta^2\right)\alpha\right)} \qquad &\rm{with\ {\bf option\ I}},\\
1-\sigma\alpha+4(4m+7)\left(\zeta\nu+2c_R D\right)\left(\beta^2+\mu^2C^2\theta^2\right)\alpha^2 \quad &\rm{with\ {\bf option\ II}}.
\end{cases}
\end{equation*}
\end{theorem}

Before proving the theorem, we summarize the above theorem as the following corollary:
\begin{corollary}
Let $\mathcal{M}$ be a Riemannian manifold whose curvature is lower-bounded and $w^* \in \mathcal{M}$ be a nondegenerate local minimizer of $f$.
Suppose Assumptions \ref{assump:T_isometry} and \ref{assump:f_C2}--\ref{Assump:ret_convex}, and let $\alpha$ be a positive number.
If $\alpha$ is sufficiently small, for any sequence $\{\tilde{w}^s\}$ generated by Algorithm \ref{Alg:R-SVRG} with a fixed step size $\alpha_t^s := \alpha$ and $m_s := m$ converging to $w^*$, there exists $K >0$ such that for all $s > K$,
\begin{equation*}
\mathbb{E}\left[\left({\rm dist}\left(\tilde{w}^s,w^*\right)\right)^2\right] \leq
\delta
\mathbb{E}\left[\left({\rm dist}\left(\tilde{w}^{s-1},w^*\right)\right)^2\right],
\end{equation*}
where $\delta$ is a constant in $(0, 1)$.
\end{corollary}

\begin{proof}[Proof of Theorem 5.14]
Since the function $x/\tanh(x)$ on $x$ monotonically increases in $[0, \infty)$, we have, from Lemma \ref{LemZhang},
\begin{equation*}
a^2 \le \zeta b^2 + c^2 - 2bc \cos(A)
\end{equation*}
for any geodesic triangle in $\Omega$ with side lengths $a$, $b$, and $c$, since $\sqrt{|\kappa|}c \le \sqrt{|\kappa|}D$.
Then, conditioned on $w_{t-1}^s$, the expectation of the distance between $w_{t}^s$ and $w^{*}$ with respect to the random choice of $i_t^s$ is evaluated by considering the geodesic triangle with $w_{t-1}^s$, $w^*$, and $w_t^s$ in $\Omega$ as
\begin{align*}
& \mathbb{E}_{i_t^s}\left[\left({\rm dist}\left(w_{t}^s, w^{*}\right)\right)^2\right] \\
\le & \mathbb{E}_{i_t^s}\left[\zeta\left({\rm dist}\left(w_{t-1}^s, w_t^s\right)\right)^2+\left({\rm dist}\left(w_{t-1}^s, w^*\right)\right)^2 - 2 \left\langle{\rm Exp}_{w_{t-1}^s}^{-1}\left(w_{t}^s\right), {\rm Exp}_{w_{t-1}^s}^{-1}\left(w^*\right)\right\rangle_{w_{t-1}^s}\right].
\end{align*}
Here, we have the following from Assumption~\ref{Assump:ret_exp}:
\begin{eqnarray*}
&& -\left\langle {\rm Exp}_{w_{t-1}^s}^{-1}\left(w_t^s\right), {\rm Exp}_{w_{t-1}^s}^{-1}\left(w^*\right)\right\rangle_{w_{t-1}^s}\\
& = &\left\langle -\alpha\xi_t^s - {\rm Exp}_{w_{t-1}^s}^{-1}\left(w_t^s\right), {\rm Exp}_{w_{t-1}^s}^{-1}\left(w^*\right)\right\rangle_{w_{t-1}^s} - \left\langle-\alpha\xi_t^s, {\rm Exp}_{w_{t-1}^s}^{-1}\left(w^*\right)\right\rangle_{w_{t-1}^s}\\
& \le & \left\|-\alpha\xi_t^s - {\rm Exp}_{w_{t-1}^s}^{-1}\left(w_t^s\right)\right\|_{w_{t-1}^s} \left\|{\rm Exp}_{w_{t-1}^s}^{-1}\left(w^*\right)\right\|_{w_{t-1}^s} - \left\langle-\alpha\xi_t^s, {\rm Exp}_{w_{t-1}^s}^{-1}\left(w^*\right)\right\rangle_{w_{t-1}^s}\\
& \le & c_R\left\|\alpha\xi_t^s\right\|_{w_{t-1}^s}^2\left\|{\rm Exp}_{w_{t-1}^s}^{-1}(w^*)\right\|_{w_{t-1}^s} - \left\langle-\alpha\xi_t^s, {\rm Exp}_{w_{t-1}^s}^{-1}(w^*)\right\rangle_{w_{t-1}^s}\\
& \le & c_RD\left\|\alpha\xi_t^s\right\|_{w_{t-1}^s}^2 - \left\langle-\alpha\xi_t^s, {\rm Exp}_{w_{t-1}^s}^{-1}\left(w^*\right)\right\rangle_{w_{t-1}^s},
\end{eqnarray*}
where the relation $\|{\rm Exp}_{w_{t-1}^s}^{-1}(w^*)\|_{w_{t-1}^s} = {\rm dist}(w_{t-1}^s, w_*) \le D$ is incorporated.
It follows that
\begin{eqnarray*}
&& \mathbb{E}_{i_t^s}\left[\left({\rm dist}\left(w_{t}^s, w^{*}\right)\right)^2 - \left({\rm dist}\left(w_{t-1}^s, w^*\right)\right)^2\right]\\
& \le & \mathbb{E}_{i_t^s}\left[\zeta\left({\rm dist}\left(w_{t-1}^s, w_{t}^s\right)\right)^2 - 2 \left\langle -\alpha \xi_t^s, {\rm Exp}_{w_{t-1}^s}^{-1}\left(w^*\right)\right\rangle_{w_{t-1}^s} + 2c_RD\left\|\alpha\xi_t^s\right\|_{w_{t-1}^s}^2\right] \\
& \leq & \mathbb{E}_{i_t^s}\left[\left(\zeta\nu+2c_R D\right)\left\|\alpha \xi_t^s\right\|_{w_{t-1}^s}^2\right] + 2\alpha\left\langle\gradf\left(w_{t-1}^s\right), {\rm Exp}_{w_{t-1}^s}^{-1}\left(w^*\right)\right\rangle_{w_{t-1}^s},
\end{eqnarray*}
where the last inequality follows from $\mathbb{E}_{i_t^s}[\xi_t^s] = \gradf(w_{t-1}^{s})$.
Lemma \ref{Lem:mean-value}, together with the relation $f(w^*) \le f(w_{t-1}^s)$, yields
\begin{align*}
\left\langle \gradf\left(w_{t-1}^s\right), {\rm Exp}_{w_{t-1}^s}^{-1}\left(w^*\right)\right\rangle_{w_{t-1}^s} \le& -\frac{\sigma}{2} \left\|{\rm Exp}_{w_{t-1}^s}^{-1}\left(w^*\right)\right\|_{w_{t-1}^s}^2 \\
=& -\frac{\sigma}{2} \left({\rm dist}\left(w_{t-1}^s, w^*\right)\right)^2.
\end{align*}
We thus obtain, by Lemma \ref{AppenLem:UpperBoundVariance},
\begin{align}
& \mathbb{E}_{i_t^s}\left[\left({\rm dist}\left(w_{t}^s, w^{*}\right)\right)^2 - \left({\rm dist}\left(w_{t-1}^s, w^*\right)\right)^2\right] \notag\\
 \le & \mathbb{E}_{i_t^s}\left[\left(\zeta\nu+2c_R D\right)\left\|\alpha\xi_t^s\right\|_{w_{t-1}^s}^2  - \sigma\alpha\left({\rm dist}\left(w_{t-1}^s, w^*\right)\right)^2\right] \notag\\
 \le & \mathbb{E}_{i_t^s}\left[4\left(\zeta\nu+2c_R D\right)\alpha^2\left(\beta^2+\mu^2C^2\theta^2\right)\left(7\left({\rm dist}\left(w_{t-1}^s, w^*\right)\right)^2 +  4\left({\rm dist}\left(\tilde{w}^{s-1}, w^*\right)\right)^2\right)\right. \notag\\
& \left. \qquad- \sigma\alpha\left({\rm dist}\left(w_{t-1}^s, w^*\right)\right)^2\right] \notag\\
= & \alpha\left(28\left(\zeta\nu+2c_R D\right)\alpha\left(\beta^2+\mu^2C^2\theta^2\right)-\sigma\right)\left({\rm dist}\left(w_{t-1}^s, w^*\right)\right)^2 \notag\\
\label{Eq:forCorollary}
 &+ 16\left(\zeta\nu+2c_R D\right)\alpha^2\left(\beta^2+\mu^2C^2\theta^2\right)\left({\rm dist}\left(\tilde{w}^{s-1}, w^*\right)\right)^2.
\end{align}
It follows for the unconditional expectation operator $\mathbb{E}$ that
\begin{eqnarray}
&& \mathbb{E}\left[\left({\rm dist}\left(w_{t}^s, w^{*}\right)\right)^2\right] - \mathbb{E}\left[\left({\rm dist}\left(w_{t-1}^s, w^*\right)\right)^2\right]\notag\\
& \le & \alpha\left(28\left(\zeta\nu+2c_R D\right)\alpha\left(\beta^2+\mu^2C^2\theta^2\right)-\sigma\right)\mathbb{E}\left[\left({\rm dist}\left(w_{t-1}^s, w^*\right)\right)^2\right]\notag\\
&& + 16\left(\zeta\nu+2c_R D\right)\alpha^2\left(\beta^2+\mu^2C^2\theta^2\right)\mathbb{E}\left[\left({\rm dist}\left(\tilde{w}^{s-1}, w^*\right)\right)^2\right]. \label{Eq:forOptions0}
\end{eqnarray}
Summing \eqref{Eq:forOptions0} over $t=1, 2,\ldots, m$ of the inner loop on the $s$-th epoch, we have 
\begin{align}
&\mathbb{E}\left[\left({\rm dist}\left(w_{m}^s, w^{*}\right)\right)^2\right] - \mathbb{E}\left[\left({\rm dist}\left(w_{0}^s,w^{*}\right)\right)^2\right] \notag\\
\leq & \alpha\left(28\left(\zeta\nu+2c_R D\right)\alpha\left(\beta^2+\mu^2C^2\theta^2\right)-\sigma\right)
\sum_{t=1}^m \mathbb{E}\left[\left({\rm dist}\left(w_{t-1}^s,w^*\right)\right)^2\right] \notag\\
& + 16m\left(\zeta\nu+2c_R D\right)\alpha^2\left(\beta^2+\mu^2C^2\theta^2\right) \mathbb{E}\left[\left({\rm dist}\left(\tilde{w}^{s-1},w^*\right)\right)^2\right]. \label{Eq:forOptions}
\end{align}
Hence, with {\bf option II}, where we compute $\tilde{w}^s$ as $\tilde{w}^s = w_m^s$, the facts $w_0^s = \tilde{w}^{s-1}$ and $\alpha(28(\zeta\nu+2c_R D)\alpha(\beta^2+\mu^2C^2\theta^2)-\sigma)<0$ imply that
\begin{eqnarray*}
&&\mathbb{E}\left[\left({\rm dist}\left(\tilde{w}^s, w^{*}\right)\right)^2\right] \\
&\le& \left(1-\sigma\alpha+4(4m+7)\left(\zeta\nu+2c_R D\right)\left(\beta^2+\mu^2C^2\theta^2\right)\alpha^2\right)\mathbb{E}\left[\left({\rm dist}\left(\tilde{w}^{s-1}, w^*\right)\right)^2\right].
\end{eqnarray*}

On the other hand, if we use {\bf option I}, where $\tilde{w}^s = g_m(w_1^s, w_2^s,\dots, w_m^s)$, rearranging \eqref{Eq:forOptions} yields
\begin{eqnarray*}
&& \alpha\left(\sigma-28\left(\zeta\nu+2c_R D\right)\alpha\left(\beta^2+\mu^2C^2\theta^2\right)\right) \sum_{t=1}^{m} \mathbb{E}\left[\left({\rm dist}\left(w_{t}^s,w^*\right)\right)^2\right] \\
& = & \alpha\left(\sigma-28\left(\zeta\nu+2c_R D\right)\alpha\left(\beta^2+\mu^2C^2\theta^2\right)\right)\\ && \times \mathbb{E}\left[\sum_{t=0}^{m-1} \left({\rm dist}\left(w_{t}^s,w^*\right)\right)^2 + \left({\rm dist}\left(w_{m}^s,w^*\right)\right)^2 - \left({\rm dist}\left(w_{0}^s,w^*\right)\right)^2 \right]\\
& \leq & \mathbb{E}\left[\left({\rm dist}\left(w_{0}^s,w^*\right)\right)^2 - \left({\rm dist}\left(w_{m}^s,w^*\right)\right)^2\right.\\ 
&&\quad+16m\left(\zeta\nu+2c_R D\right)\alpha^2\left(\beta^2+\mu^2C^2\theta^2\right) \left({\rm dist}\left(w_{0}^s,w^*\right)\right)^2 \\
& & \left.\quad -\alpha\left(\sigma-28\left(\zeta\nu+2c_R D\right)\alpha\left(\beta^2+\mu^2C^2\theta^2\right)\right) \left(\left({\rm dist}\left(w_{0}^s,w^*\right)\right)^2 - \left({\rm dist}\left(w_{m}^s,w^*\right)\right)^2\right)\right]\\
& \leq & \left(1 - \alpha\left(\sigma-28\left(\zeta\nu+2c_R D\right)\alpha\left(\beta^2+\mu^2C^2\theta^2\right)\right) \right.\\
&& \left. \qquad\qquad\qquad + 16m\left(\zeta\nu+2c_R D\right)\alpha^2\left(\beta^2+\mu^2C^2\theta^2\right)\right)\mathbb{E}\left[\left({\rm dist}\left(w_{0}^s,w^*\right)\right)^2\right]\\
& \leq & \left(1+ 16m\left(\zeta\nu+2c_R D\right)\alpha^2\left(\beta^2+\mu^2C^2\theta^2\right)\right) \mathbb{E}\left[\left({\rm dist}\left(\tilde{w}^{s-1},w^*\right)\right)^2\right].
\end{eqnarray*}
Using $\tilde{w}^{s} = g_m(w_1^s, w_2^s,\ldots,w_m^s)$ and Lemma \ref{AppenLem:KarcherMeanDistance}, we obtain
\begin{equation*}
\mathbb{E}\left[\left({\rm dist}\left(\tilde{w}^s,w^*\right)\right)^2\right] \leq
\frac{4\left(1+ 16m\left(\zeta\nu+2c_R D\right)\left(\beta^2+\mu^2C^2\theta^2\right)\alpha^2\right)}{m\alpha \left(\sigma-28\left(\zeta\nu+2c_R D\right)\left(\beta^2+\mu^2C^2\theta^2\right)\alpha\right)} 
\mathbb{E}\left[\left({\rm dist}\left(\tilde{w}^{s-1},w^*\right)\right)^2\right].
\end{equation*}
This completes the proof.
\end{proof}
In the above theorem, we note that, from the definitions of $\beta$ and $\sigma$, $\beta$ can be chosen to be arbitrarily large and $\sigma$ can be chosen to be arbitrarily small.
Therefore, $\alpha = \sigma / 56(\zeta\nu+2c_R D)(\beta^2+\mu^2C^2\theta^2)$, e.g., satisfies $0 < \alpha (\sigma-28(\zeta\nu+2c_R D)(\beta^2+\mu^2C^2\theta^2)\alpha) < 1 $ for sufficiently large $\beta$ and small $\sigma$.

In fact, $\alpha$ satisfying the condition always exists for any values of $\beta$ and $\sigma$.
Let $\beta' := 28(\zeta\nu+2c_R D)(\beta^2+\mu^2C^2\theta^2)$.
We can analyze the inequality $0<\alpha(\sigma-28(\zeta\nu+2c_R D)(\beta^2+\mu^2C^2\theta^2)\alpha)<1$, which is expressed as $0<\alpha(\sigma-\beta'\alpha)<1$, with the condition $\alpha>0$ more specifically as
\begin{equation}
\label{Eq:alpha}
\begin{cases}
\displaystyle
0 < \alpha < \frac{\sigma}{\beta'} \qquad & \text{if}\quad \sigma^2 - 4\beta' < 0,\\
\displaystyle
0 < \alpha < \frac{\sigma}{2\beta'},\ \frac{\sigma}{2\beta'} < \alpha < \frac{\sigma}{\beta'} \qquad & \text{if}\quad \sigma^2 - 4\beta' = 0,\\
\displaystyle
0 < \alpha < \frac{\sigma - \sqrt{\sigma^2-4\beta'}}{2\beta'},\ \frac{\sigma + \sqrt{\sigma^2-4\beta'}}{2\beta'} < \alpha < \frac{\sigma}{\beta'} \qquad & \text{if}\quad \sigma^2 - 4 \beta' >0.
\end{cases}
\end{equation}

Furthermore, we can show that the coefficient $\delta$ on the right-hand side of \eqref{Eq:LinearConvergence}, which can be written as $4(7+4m\beta'\alpha^2)/7m\alpha(\sigma-\beta'\alpha) =: r_1(\alpha)$ for {\bf Option I} and $1-\sigma\alpha + (1+4m/7)\beta'\alpha^2 =: r_2(\alpha)$ for {\bf Option II}, is less than $1$ when $m$ is sufficiently large and $\alpha$ is appropriately chosen.
If $\alpha$ is fixed, $r_1(\alpha) \to 16\beta'\alpha/7(\sigma-\beta'\alpha)$ as $m \to \infty$, which is not necessarily less than $1$, and $r_2(\alpha) \to \infty$ as $m \to \infty$.
Thus, we again need to specifically analyze an appropriate value of $\alpha$, which should depend on $m$.

For {\bf Option I}, by calculating the derivative $r_1'(\alpha)$ of $r_1(\alpha)$ on $\alpha$,
we can show that $r_1(\alpha)$ takes the minimum value at
\begin{equation*}
\alpha = \frac{-7\beta' + \sqrt{49\beta'^2+28m\beta'\sigma^2}}{4m\beta'\sigma} =: \alpha_*,
\end{equation*}
which satisfies \eqref{Eq:alpha} when $m$ is sufficiently large, since $\lim_{m\to\infty}\alpha_* = 0$.
Note that we have $r_1'(\alpha_*) = 0$, which yields $4m\beta' \sigma\alpha_*^2 + 14\beta'\alpha_*-7\sigma = 0$.
This relation gives the minimum value $r_1(\alpha_*)$ as
\begin{equation*}
r_1(\alpha_*) = \frac{32\left(\sigma-\beta'\alpha_*\right)}{2\left(7\beta'+2m\sigma^2\right)\alpha_*-7\sigma} \to 0 \quad (m \to \infty),
\end{equation*}
where we have used the facts that $\lim_{m\to \infty} \alpha_* = 0$ and $\lim_{m\to \infty} m\alpha_* =\infty$.
A simpler choice of $\alpha=1/\sqrt{m}$ also makes $r_1(\alpha)$ less than $1$ if $m$ is sufficiently large since
\begin{equation*}
r_{1}\left(\frac{1}{\sqrt{m}}\right) = \frac{4\left(7+4\beta'\right)}{7\left(\sigma\sqrt{m}-\beta'\right)}\to 0 \quad (m\to \infty).
\end{equation*}
Although $\alpha = 1/\sqrt{m}$ does not achieve the best rate attained by $\alpha=\alpha_*$, this choice is practical because we do not know the exact values of $\sigma$, $\beta$, or $\alpha_*$ in general.

A similar discussion can be applied to the case of {\bf Option II}.
In this case, it is clear that a sufficiently small $\alpha$ satisfies \eqref{Eq:alpha} and $r_2(\alpha) < 1$.
Furthermore, if $\sigma^2-4\beta'\leq 0$, $\alpha_* = \sigma / 2(1+4m/7)\beta'$ satisfies \eqref{Eq:alpha} and attains the minimum value of $r_2$ as
\begin{equation*}
r_2(\alpha_*) = 1-\frac{7\sigma^2}{4(4m+7)\beta'} < 1.
\end{equation*}
Furthermore, this $\alpha_*$ satisfies \eqref{Eq:alpha} if $\sigma^2-4\beta' > 0$ and $m$ is sufficiently large.

We have thus shown that a local linear convergence rate is achieved under an appropriate fixed step size if $m$ is sufficiently large, which is the same as standard SVRG in Euclidean space (for nonconvex problems).
We can also analyze the rate with decaying step sizes
$\alpha_0^s>\alpha_1^s>\dots>\alpha_m^s$ (at the $s$-th epoch) as
\begin{equation*}
\frac{4(1+ 16m(\zeta\nu+2c_R D)(\beta^2+\mu^2C^2\theta^2)(\alpha_0^s)^2)}{m\alpha_m^s (\sigma-28(\zeta\nu+2c_R D)(\beta^2+\mu^2C^2\theta^2)\alpha_0^s)}
\end{equation*}
for {\bf Option~I}.
This is larger than
\begin{equation*}
\frac{4(1+ 16m(\zeta\nu+2c_R D)(\beta^2+\mu^2C^2\theta^2)(\alpha_0^s)^2)}{m\alpha_0^s (\sigma-28(\zeta\nu+2c_R D)(\beta^2+\mu^2C^2\theta^2)\alpha_0^s)},
\end{equation*}
which is the coefficient in \eqref{Eq:LinearConvergence} with the fixed step size $\alpha=\alpha_0^s$.
A similar discussion can be applied to the case of {\bf Option II}.
Consequently, using decaying step sizes also yields a local convergence, but at a worse rate than with a fixed step size.
Both the above guarantees are quite similar to those available for batch gradient algorithms on manifolds.
This setup, i.e., hybrid step sizes, follows our two convergence analyses, which first require decaying step sizes to approach a neighborhood of a local minimum and use a fixed step size to achieve faster linear convergence near the solution.
As mentioned earlier, we guarantee global convergence and local linear convergence even if we use decaying step sizes from beginning to end. Therefore, this method is an improved version of decaying step size.
Theoretical analysis of the switching between decaying and fixed step sizes is left for future work.

We make one more remark.
Although this subsection describes the local convergence analysis with an objective function that is strictly retraction-convex in a sufficiently small neighborhood of a local minimizer, it is also applicable for a global convergence analysis if we assume that the function is globally strictly convex, as in other studies.
As we have already discussed, the rate in Theorem \ref{Thm:LocalConvergence} can
lead to discussions on global iteration complexity.
The key aspect of our local convergence analysis is that we do not have to assume global convexity to attain the local linear convergence rate.

\subsection{Local convergence rate analysis with exponential mapping and parallel translation}
In this subsection, we present a local convergence rate analysis of the R-SVRG algorithm with exponential mapping and parallel translation along the geodesics.
This is a special case of the previous subsection where exponential mapping and parallel translation are chosen as retraction and vector transport, respectively.
However, in this particular case, we can obtain a stricter rate than in a general case.
Since the results are obtained by a similar discussion, we give a sketch of the proofs.

We obtain the following result as a corollary of the proof of Lemma \ref{AppenLem:UpperBoundVariance}, with $R={\rm Exp}$ and $\mathcal{T} = P$.
\begin{corollary} 
\label{Cor:UpperBoundVariance}
Suppose the conditions in Lemma~\ref{AppenLem:UpperBoundVariance} and
consider Algorithm~\ref{Alg:R-SVRG} with $R={\rm Exp}$ and $\mathcal{T} = P$, i.e., the exponential mapping and parallel translation case.
Let $\beta$ be a constant in Lemma~\ref{AppenLem:UpperBoundVariance}.
Then, the upper bound of $\mathbb{E}_{i_t^s}[\| \xi_t^s \|_{w_{t-1}^s}^2]$ is given by
\begin{eqnarray}
\label{Eq:UpperBoundVariance}
	\mathbb{E}_{i_t^s}\left[\left\| \xi_t^s \right\|_{w_{t-1}^s}^2\right] &\leq &
\beta^2 \left(14\left({\rm dist}\left(w_{t-1}^s,w^*\right)\right)^2 + 8\left({\rm dist}\left(\tilde{w}^{s-1},w^*\right)\right)^2\right).
\end{eqnarray}
\end{corollary}

\begin{proof}
Putting $R={\rm Exp}$ and $\mathcal{T} = P$ in the middle of the proof of Lemma \ref{AppenLem:UpperBoundVariance}, which is in Appendix~\ref{Appendix},
we obtain
\begin{eqnarray*}
& & \hspace*{-1cm}\mathbb{E}_{i_t^s}\left[\left\| \xi_t^s \right\|_{w_{t-1}^s}^2\right]   \nonumber\\
& \le & 2\mathbb{E}_{i_t^s}\left[ \left\| \gradf_{i_t^s}\left(w_{t-1}^{s}\right) - P_{\gamma}^{w_{t-1}^{s} \leftarrow w^{*}}\left(\gradf_{i_t^s}\left(w^{*}\right)\right)  \right\|_{w_{t-1}^s}^2 \right]  \nonumber\\
&&+ 2\mathbb{E}_{i_t^s}\left[ \left\| P_{\gamma}^{w_{t-1}^s \leftarrow \tilde{w}^{s-1}} \left(\gradf_{i_t^s}\left(\tilde{w}^{s-1}\right) \right) - P_{\gamma}^{w_{t-1}^{s} \leftarrow w^{*}}\left(\gradf_{i_t^s}\left(w^{*}\right)\right) \right\|_{w_{t-1}^s}^2 \right]  \nonumber\\
&& - 2 \left\| P_{\gamma}^{w_{t-1}^s \leftarrow \tilde{w}^{s-1}} \left(\gradf\left(\tilde{w}^{s-1}\right) \right)\right\|_{w_{t-1}^s}^2.  \nonumber
\end{eqnarray*}
In a similar manner to Lemma \ref{AppenLem:UpperBoundVariance}, we have
\allowdisplaybreaks[1]
\begin{eqnarray*}
& & \mathbb{E}_{i_t^s}\left[\left\| \xi_t^s \right\|_{w_{t-1}^s}^2\right]   \nonumber\\
& \le & 
2\mathbb{E}_{i_t^s}\left[ \left\| \gradf_{i_t^s}\left(w_{t-1}^{s}\right) - P_{\gamma}^{w_{t-1}^{s} \leftarrow w^{*}}\left(\gradf_{i_t^s}\left(w^{*}\right)\right)  \right\|_{w_{t-1}^s}^2 \right]
 \nonumber\\
&&
+ 2\mathbb{E}_{i_t^s}\left[ \biggl\| P_{\gamma}^{w_{t-1}^s \leftarrow \tilde{w}^{s-1}} \left(\gradf_{i_t^s}\left(\tilde{w}^{s-1}\right) \right) - \gradf_{i_t^s}\left(w_{t-1}^{s}\right) \right.\nonumber\\
&& \left.  \qquad \qquad \qquad \qquad \qquad + \gradf_{i_t^s}\left(w_{t-1}^{s}\right) - P_{\gamma}^{w_{t-1}^{s} \leftarrow w^{*}}\left(\gradf_{i_t^s}\left(w^{*}\right)\right) \biggr\|_{w_{t-1}^s}^2 \right] \nonumber
\\
& \le & 
2\mathbb{E}_{i_t^s}\left[ \left\| \gradf_{i_t^s}(w_{t-1}^{s}) - P_{\gamma}^{w_{t-1}^{s} \leftarrow w^{*}}\left(\gradf_{i_t^s}\left(w^{*}\right)\right)  \right\|_{w_{t-1}^s}^2 \right]
  \nonumber\\
&&
+ 4\mathbb{E}_{i_t^s}\left[ \left\| P_{\gamma}^{w_{t-1}^s \leftarrow \tilde{w}^{s-1}} \left(\gradf_{i_t^s}\left(\tilde{w}^{s-1}\right) \right) - \gradf_{i_t^s}\left(w_{t-1}^{s}\right)\right\|_{w_{t-1}^s}^2 \right]
\nonumber\\
&&
 + 4\mathbb{E}_{i_t^s}\left[ \left\| \gradf_{i_t^s}\left(w_{t-1}^{s}\right) - P_{\gamma}^{w_{t-1}^{s} \leftarrow w^{*}}\left(\gradf_{i_t^s} \left(w^{*}\right)\right) \right\|_{w_{t-1}^s}^2 \right] \nonumber
\\
& \leq & 
\beta^2 (6({\rm dist}(w_{t-1}^s,w^*))^2 + 4({\rm dist}(\tilde{w}^{s-1},w_{t-1}^s))^2  )
 \nonumber\\
& \le & 
\beta^2 \left(6\left({\rm dist}\left(w_{t-1}^s,w^*\right)\right)^2 + 4\left({\rm dist}\left(\tilde{w}^{s-1},w^*\right)+{\rm dist}\left(w^*,w_{t-1}^s\right)\right)^2\right)
 \nonumber\\
& \le & 
\beta^2 \left(6\left({\rm dist}\left(w_{t-1}^s,w^*\right)\right)^2 + 8\left({\rm dist}\left(\tilde{w}^{s-1},w^*\right)\right)^2+8\left({\rm dist}\left(w^*,w_{t-1}^s\right)\right)^2\right)
 \nonumber\\
& = & 
\beta^2 \left(14\left({\rm dist}\left(w_{t-1}^s,w^*\right)\right)^2 + 8\left({\rm dist}\left(\tilde{w}^{s-1},w^*\right)\right)^2\right).
\end{eqnarray*}
This completes the proof.
\end{proof}

\begin{corollary}
Suppose the conditions in Theorem \ref{Thm:LocalConvergence}, except that a positive number $\alpha$ satisfies $0<\alpha(\sigma-14\zeta\beta^2\alpha)<1$, and consider Algorithm \ref{Alg:R-SVRG} with a fixed step size $\alpha_t^s := \alpha$ for the exponential mapping and parallel translation case.
For any sequence $\{\tilde{{w}}^s\}$ generated by the algorithm, there exists $K>0$ such that for all $s>K$,
\begin{equation*}
\mathbb{E}\left[\left({\rm dist}\left(\tilde{w}^s,w^*\right)\right)^2\right] \leq
\delta_0
\mathbb{E}\left[\left({\rm dist}\left(\tilde{w}^{s-1},w^*\right)\right)^2\right],
\end{equation*}
where
\begin{equation*}
\delta_0 :=
\begin{cases}
\displaystyle\frac{4\left(1+ 8 m \zeta\beta^2\alpha^2\right)}{m \alpha \left(\sigma - 14 \zeta\beta^2\alpha\right)} 
 \qquad &\rm{with\ {\bf option\ I}},\\
1-\sigma\alpha+(8m+14)\zeta\beta^2\alpha^2 \quad &\rm{with\ {\bf option\ II}}.
\end{cases}
\end{equation*}
\end{corollary}

\begin{proof}
Note that constants in Theorem~\ref{Thm:LocalConvergence} are $c_R = \theta = 0$ and $\mu = \nu =1$ in this case.
By using \eqref{Eq:UpperBoundVariance} instead of \eqref{Append_Eq:UpperBoundVariance}, we obtain
\begin{eqnarray*}
&& \mathbb{E}_{i_t^s}\left[\left({\rm dist}\left({w}_{t}^s, {w}^{*}\right)\right)^2 - \left({\rm dist}\left({w}_{t-1}^s, {w}^*\right)\right)^2\right]\\
& \le & \alpha\left(14\zeta\alpha\beta^2-\sigma\right)\left({\rm dist}\left({w}_{t-1}^s, {w}^*\right)\right)^2 + 8\zeta\alpha^2\beta^2\left({\rm dist}\left(\tilde{{w}}^{s-1}, {w}^*\right)\right)^2,
\end{eqnarray*}
instead of \eqref{Eq:forCorollary}.
Summing over $t=1, 2,\ldots, m$ of the inner loop on the $s$-th epoch, we have 
\begin{align*}
&\mathbb{E}\left[\left({\rm dist}\left({w}_{m}^s, {w}^{*}\right)\right)^2\right] - \mathbb{E}\left[\left({\rm dist}\left({w}_{0}^s,{w}^{*}\right)\right)^2\right] \\
\leq & \alpha (14\zeta\alpha\beta^2 - \sigma) 
\sum_{t=1}^m \mathbb{E}\left[\left({\rm dist}\left({w}_{t-1}^s,{w}^*\right)\right)^2\right]
 + 8 m \zeta\alpha^2  \beta^2 \mathbb{E}\left[\left({\rm dist}\left(\tilde{{w}}^{s-1},{w}^*\right)\right)^2\right].
\end{align*}
A similar discussion as in the proof of Theorem \ref{Thm:LocalConvergence} yields the claimed convergence rates.
\end{proof}

\section{Numerical comparisons}
\label{Sec:NumericalComparison}

This section compares the performance of R-SVRG(+) (with {\bf option II}) with that of the Riemannian extension of SGD, i.e., R-SGD, where the Riemannian stochastic gradient algorithm uses $\gradf_{i_t^s}(w_{t-1}^{s})$ instead of $\xi_t^s$ in (\ref{Eq:R-SVRG-Grad-paralleltrans}).
We also make a comparison with R-SD, which is the Riemannian steepest descent algorithm with backtracking line search \cite[Chapters~4]{Absil_OptAlgMatManifold_2008}.
We consider both \emph{fixed} step size and \emph{decaying} step size sequences.
The decaying step size sequence uses the decay $\alpha_k = \alpha_0(1+ \alpha_0 \lambda \lfloor k/m_s \rfloor)^{-1}$, where $k$ is the number of iterations. We select some values of $\alpha_0$ and consider three values of $\lambda$ ($10^{-1}, 10^{-2}, \text{ and } 10^{-3}$).
In addition, since the global convergence analysis needs a decaying step size condition and the local convergence rate analysis holds for a fixed step size (Sections \ref{Sec:GlobalAnalysis} and \ref{Sec:LocalAnalysis}), we consider a {\it hybrid} step size sequence that follows the decaying step size before the $s_{\mathrm{TH}}$ epoch and subsequently switches to a fixed step size.
All the experiments use $m_s=5N$ by following \cite{Johnson_NIPS_2013_s}.
In all the figures, the $x$-axis represents the computational cost measured by the number of gradient computations divided by $N$.
The algorithms are initialized randomly and are terminated when either the stochastic gradient norm is below $10^{-8}$ or the number of iterations exceeds a predefined threshold.
It should be noted that all the results except those of R-SD are the best-tuned results.
All the simulations were performed in MATLAB on a 2.6 GHz Intel Core i7 machine with 16 GB RAM. 
Hereinafter, this paper addresses three problems on the SPD manifold and the Grassmann manifold.
In all the problems, full gradient methods, e.g., the steepest descent algorithm, become prohibitively computationally expensive when $N$ is extremely large.
The stochastic gradient approach is a promising way to achieve scalability.

\subsection{Problem on the SPD manifold and simulation results}

We first consider the Riemannian centroid problem on the SPD manifold.

{\bf SPD manifold $\mathcal{S}_{++}^d$ and optimization tools.} We designate the space of $d \times d$ SPD matrices as the SPD manifold, $\mathcal{S}_{++}^d$. If we endow $\mathcal{S}_{++}^d$ with the affine-invariant Riemannian metric (AIRM) \cite{Pennec_IJCV_2006} defined by
$\langle \xi_{\scriptsize \mat{X}}, \eta_{\scriptsize \mat{X}} \rangle_{\scriptsize \mat{X}}
=\rm{trace}(\xi_{\scriptsize \mat{X}} \mat{X}^{-1} \eta_{\scriptsize \mat{X}} \mat{X}^{-1})$ for $\xi_{\scriptsize \mat{X}}, \eta_{\scriptsize \mat{X}} \in T_{\scriptsize \mat{X}}\mathcal{S}_{++}^d$ at $\mat{X} \in \mathcal{S}_{++}^d$, the SPD manifold $\mathcal{S}_{++}^d$ forms a Riemannian manifold.
The exponential mapping is written as
\begin{equation*}
{\rm Exp}_{\scriptsize \mat{X}}(\xi_{\scriptsize \mat{X}})= \mat{X}^{1/2} \exp(\mat{X}^{-1/2} \xi_{\scriptsize \mat{X}}\mat{X}^{-1/2}) \mat{X}^{1/2}
\end{equation*}
for any $\xi_{\scriptsize \mat{X}}$ and $\mat{X}$.
The parallel translation of $\xi_{\scriptsize \mat{X}}$ along $\eta_{\scriptsize \mat{X}}$ on $\mathcal{S}_{++}^d$  is given by $P_{\eta_{\tiny \mat{X}}}(\xi_{\scriptsize \mat{X}}) = \mat{X}^{1/2} \mat{Y} \mat{X}^{-1/2}\xi_{\scriptsize \mat{X}}\mat{X}^{-1/2} \mat{Y} \mat{X}^{1/2}$, where $\mat{Y}= \exp(\mat{X}^{-1/2}\eta_{\scriptsize \mat{X}}\mat{X}^{-1/2}/2)$. The logarithm map of $\mat{Y}$ at $\mat{X}$ is described as 
\begin{equation*}
{\rm Log}_{\scriptsize \mat{X}}(\mat{Y}) = \mat{X}^{1/2} \log (\mat{X}^{-1/2}\mat{Y}\mat{X}^{-1/2})\mat{X}^{1/2} = \log(\mat{Y}\mat{X}^{-1})\mat{X}.
\end{equation*}

The exponential mapping and the parallel translation above are computationally expensive. Therefore, an efficient retraction is proposed as \cite{JeurisVV_2012};
\begin{eqnarray}
	\label{Eq:SPD_retraction}
	R_{\scriptsize \mat{X}}\left(\xi_{\scriptsize \mat{X}}\right) = \mat{X}+\xi_{\scriptsize \mat{X}}+\frac{1}{2}\xi_{\scriptsize \mat{X}} \mat{X}^{-1} \xi_{\scriptsize \mat{X}}. 
\end{eqnarray}
This maps $\xi_{\scriptsize \mat{X}}$ onto $\mathcal{S}_{++}^d$ for all $\xi_{\scriptsize \mat{X}} \in T_{\scriptsize \mat{X}}\mathcal{S}_{++}^d$. Huang et al. proposed an efficient isometric vector transport \cite{huang2015riemannian,Yuana_ICCS_2016_s} defined as
\begin{eqnarray}
	\label{Eq:SPD_vectortrans}
	\mathcal{T}_{S_{\eta}}\xi_{\scriptsize \mat{X}} = B_{\scriptsize \mat{Y}} B_{\scriptsize \mat{X}}^{\flat} \xi_{\scriptsize \mat{X}},
\end{eqnarray}	
where $\mat{Y}=R_{\scriptsize \mat{X}}(\xi_{\scriptsize \mat{X}})$ and $a^{\flat}$ denotes the flat of $a \in T_w \mathcal{M}$, i.e., $a^{\flat}\colon T_w \mathcal{M} \to \mathbb{R}\colon v \mapsto \langle a,v \rangle_{w}$. $B_{\scriptsize \mat{X}}$ and $B_{\scriptsize \mat{Y}}$ are the orthonormal bases of $T_{\scriptsize \mat{X}}\mathcal{S}_{++}^d$ and $T_{\scriptsize \mat{Y}}\mathcal{S}_{++}^d$, respectively, where the basis is calculated based on the Cholesky decomposition. Consequently, the implementation of our algorithm for this particular problem uses the retraction~\eqref{Eq:SPD_retraction} and the vector transport (\ref{Eq:SPD_vectortrans}), which satisfy the requirements in the convergence analyses in Sections \ref{Sec:GlobalAnalysis} and \ref{Sec:LocalAnalysis}.

{\bf Riemannian centroid problem.} We first evaluate the proposed algorithm in the Riemannian centroid problem on $\mathcal{S}_{++}^d$, which is frequently used for computer vision problems, such as visual object categorization and pose categorization \cite{Jayasumana_IEEETPAMI_2015_s}. 
Given $N$ points on $\mathcal{S}_{++}^d$ with matrix representations $\mat{X}_1,\mat{X}_2,\dots,\mat{X}_N$, the Riemannian centroid is derived from the solution to the problem
\begin{eqnarray*}
\label{Eq:RiemannianCentroid}
{\displaystyle \min_{{\scriptsize \mat{C} \in \mathcal{S}_{++}^d}}} &{\displaystyle \frac{1}{2N} \sum_{n=1}^N \left({\rm dist}\left(\mat{C}, \mat{X}_n\right)\right)^2},
\end{eqnarray*}
where ${\rm dist}(\mat{A}, \mat{B})=\| \log (\mat{A}^{-1/2}\mat{B}\mat{A}^{-1/2})\|_F$ represents the distance along the corresponding geodesic between the two points $\mat{A}, \mat{B} \in \mathcal{S}_{++}^d$ with respect to the AIRM.
The gradient of the loss function is computed as $\frac{1}{N} \sum_{n=1}^{N} -{\rm log}(\mat{X}_n \mat{C}^{-1})\mat{C}$. 

Figures \ref{fig:RiemannianCentroid_results}(a) and (b) show the results of the optimality gap and the norm of the gradient, respectively, where $N=1000$ and $d=3$. The choices of $\alpha_0$ are $\{10^{-3},2\times 10^{-3}, 4\times 10^{-3},6\times 10^{-3},8\times 10^{-3},10^{-2},2\times 10^{-2},\ldots, 10^{-1}\}$. 
$s_{\mathrm{TH}}$ and the batch size are fixed to $3$ and $1$, respectively. The maximum number of iterations is $10$ for R-SVRG(+) and $60$ for the others. The optimality gap evaluates the performance against the minimum loss, which is obtained by R-SD with high precision in advance. From the figures, R-SVRG(+) outperforms R-SGD in terms of the gradient counts and exhibits much faster convergence than R-SD as expected. 

\begin{figure}[htbp]
\begin{center}
	\begin{minipage}[t]{0.45\textwidth}
	\begin{center}
		\includegraphics[width=\textwidth]{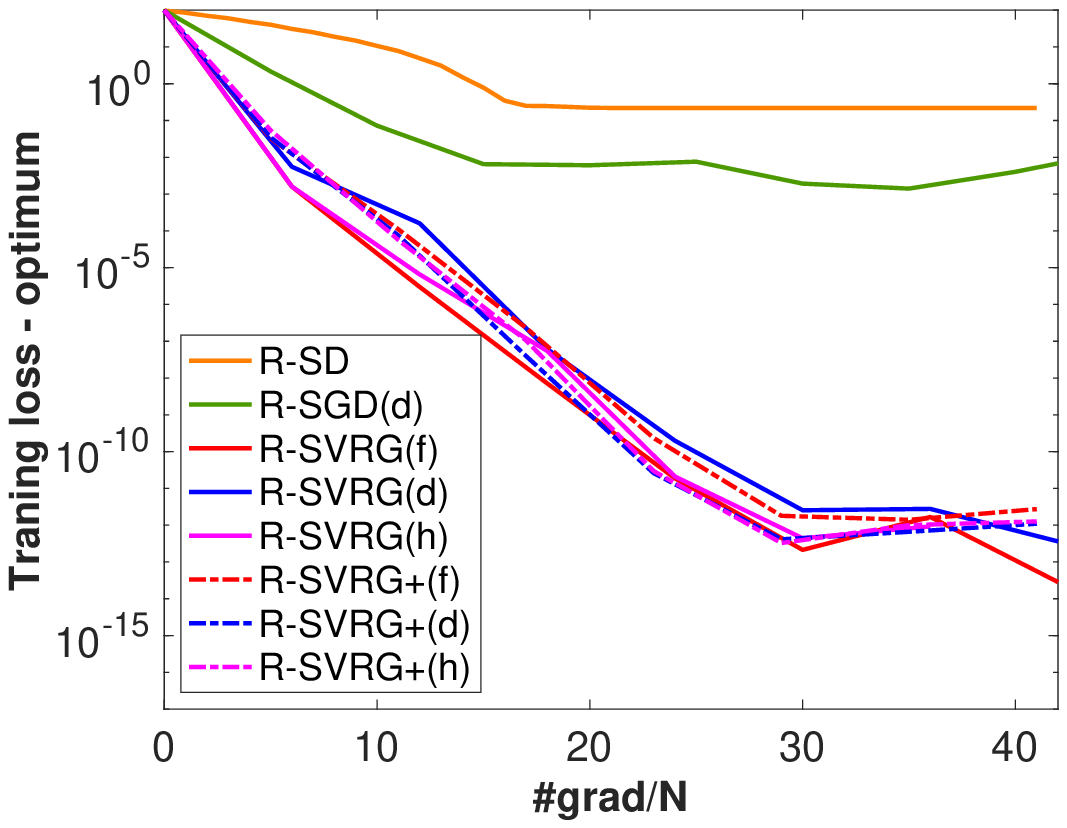}\\
		{\small (a) Optimality gap.}
	\end{center} 
	\end{minipage}
	\hspace*{1cm}
	\begin{minipage}[t]{0.45\textwidth}
	\begin{center}
		\includegraphics[width=\textwidth]{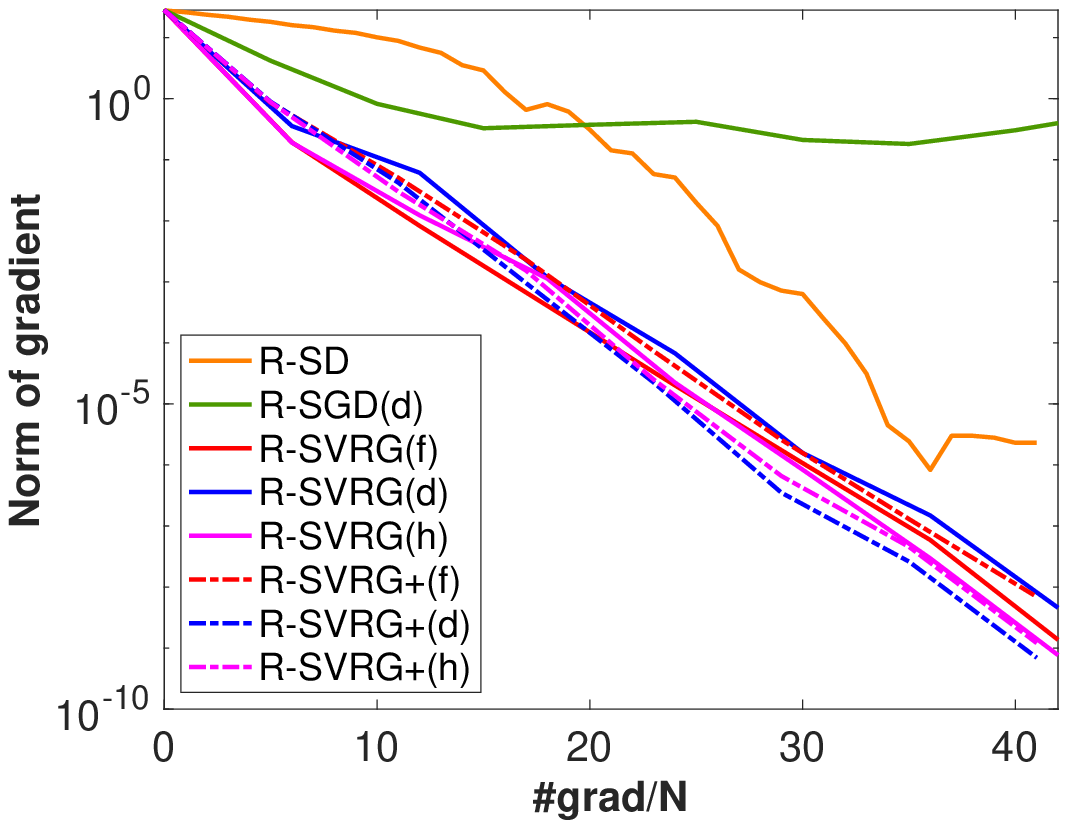}\\
		{\small (b) Norm of gradient.}
	\end{center} 
	\end{minipage}\\
	\vspace*{0.3cm}
\caption{Performance evaluations on Riemannian centroid problem.
In the legends of the figures, (f), (d), and (h) denote fixed, decaying, and hybrid step sizes, respectively.
The parameters are chosen as follows.
R-SGD(d): $\alpha_0= 0.002$, $\lambda= 0.1$; R-SVRG(f): $\alpha_0= 0.008$; R-SVRG(d): $\alpha_0= 0.1$, $\lambda= 0.1$; R-SVRG(h): $\alpha_0=0.01$, $\lambda= 0.001$; R-SVRG+(f): $\alpha_0 = 0.01$; R-SVRG+(d): $\alpha_0= 0.006$, $\lambda= 0.01$; R-SVRG+(h): $\alpha_0 = 0.004$, $\lambda= 0.001$.}
\label{fig:RiemannianCentroid_results}
\end{center}
\end{figure}

\subsection{Problems on the Grassmann manifold and simulation results}
We focus on two popular problems on the Grassmann manifold: PCA and low-rank matrix completion problems.

{\bf Grassmann manifold and optimization tools.}
An element on the Grassmann manifold is represented by a $d \times r$ orthogonal matrix \mat{U} with orthonormal columns, i.e., $\mat{U}^T\mat{U}=\mat{I}$. Two orthogonal matrices represent the same element on the Grassmann manifold if they are related by right multiplication of an $r\times r$ orthogonal matrix $\mat{O} \in \mathcal{O}(r)$, where $\mathcal{O}(r)$ is the orthogonal group of order $r$. Equivalently, an element of the Grassmann manifold is identified with a set of $d \times r$ orthogonal matrices $[\mat{U}]: =\{\mat{U}\mat{O} \mid\mat{O} \in \mathcal{O}(r)\}$.
Thus, ${\rm Gr}(r,d) :={\rm St}(r,d)/ \mathcal{O}(r)$, where ${\rm St}(r,d)$ is the Stiefel manifold, which is the set of matrices of size $d \times r$ with orthonormal columns. The Grassmann manifold has the structure of a Riemannian quotient manifold \cite[Section~3.4]{Absil_OptAlgMatManifold_2008}. 
The exponential mapping for the Grassmann manifold from $\mat{U}(0) := \mat{U} \in {\rm Gr}(r,d)$ in the direction of $\xi \in T_{\scriptsize \mat{U}(0)} {\rm Gr}(r,d)$ is given in a closed form as \cite[Section 5.4]{Absil_OptAlgMatManifold_2008}
\begin{eqnarray*}
\label{Eq:exponential_map}
\mat{U} (t) & = & [\mat{U}(0)  \mat{V}\ \  \mat{W}] 
	\left[
    		\begin{array}{c}
      		\cos t \Sigma  \\
      		\sin t \Sigma \\
    		\end{array}
	\right]
    \mat{V}^T,
\end{eqnarray*}
where $\xi=\mat{W} \Sigma \mat{V}^T$ is the rank-$r$ singular value decomposition of $\xi$. The $\cos(\cdot)$ and $\sin(\cdot)$ operations are only on the diagonal entries. 
The parallel translation of $\zeta \in T_{\scriptsize \mat{U}(0)} {\rm Gr}(r,d)$ on the Grassmann manifold along $\gamma(t)$ with $\dot \gamma(0) = \mat{W} \Sigma \mat{V}^T$ is given in a closed-form by
\begin{eqnarray*}
\label{Eq:parallel_translation}
\zeta(t) & = & \left( [\mat{U}(0) \mat{V}\ \  \mat{W}] 
	\left[
    		\begin{array}{c}
      		-\sin t \Sigma \\
      		\cos t \Sigma \\
    		\end{array}
	\right]
    \mat{W}^T + \left(\mat{I} - \mat{W}\mat{W}^T\right)
    \right) \zeta.
\end{eqnarray*}
The logarithm map of $\mat{U}(t)$ at $\mat{U}(0)$ on the Grassmann manifold is given by
\begin{eqnarray*}
\label{Eq:logarithm_map}
\xi  &=& \log_{\scriptsize \mat{U}(0)}(\mat{U}(t)) \ = \ \mat{W} \arctan(\Sigma) \mat{V}^T,
\end{eqnarray*}	
where the rank-$r$ singular value decomposition of $(\mat{U}(t) - \mat{U}(0) \mat{U}(0)^T  \mat{U}(t))\allowbreak(\mat{U}(0)^T \mat{U}(t))^{-1}$ is $\mat{W}\Sigma \mat{V}^T$. 
It should be noted that this experiment evaluates the projection-based vector transport and the QR-decomposition-based retraction, which do not satisfy the conditions in Sections \ref{Sec:GlobalAnalysis} and \ref{Sec:LocalAnalysis}, but are computationally efficient. The intention here is to show that our algorithm performs well empirically without using the specific vector transport.

{\bf The PCA problem.} Given an orthonormal matrix projector $\mat{U} \in {\rm St}(r,d)$, the PCA problem is to minimize the sum of the squared residual errors between the projected data points and the original data, as
\begin{equation}\label{Eq:PCA}
\begin{array}{ll}
{\displaystyle \min_{{\scriptsize \mat{U} \in {\rm St}(r,d)}}} & {\displaystyle \frac{1}{N}  \sum_{n=1}^N \left\| \vec{x}_n -  \mat{U}\mat{U}^T \vec{x}_n \right\|_2^2},
\end{array}
\end{equation}
where $\vec{x}_n$ is a data vector of size $d\times 1$. Problem \eqref{Eq:PCA} is equivalent to maximizing the function $\frac{1}{N} \sum_{n=1}^N \vec{x}_n^T\mat{U}\mat{U}^T\vec{x}_n$.
Here, the critical points in the space ${\rm St}(r,d)$ are not isolated because the cost function remains unchanged under the group action $\mat{U} \mapsto \mat{UO}$ for any orthogonal matrices $\mat{O}$ of size $r \times r$. Consequently, Problem \eqref{Eq:PCA} is reformulated as an optimization problem on the Grassmann manifold ${\rm Gr}(r,d)$.

Figures \ref{fig:PCA_results}(a) and (b) show the optimality gap and gradient norm, respectively, where $N=10000$, $d=20$, and $r=5$.
The choices of $\alpha_0$ are $\{10^{-3}, 2\times10^{-3}, \ldots, 10^{-2}\}$.
The minimum loss for the optimality gap evaluation is obtained by the MATLAB function {\tt pca}.
$s_{\mathrm{TH}}$ and the batch size are fixed to 5 and 10, respectively. The maximum number of iterations is $16$ for R-SVRG(+) and $100$ for the others.
From Figure \ref{fig:PCA_results}(a), we can observe that between R-SVRG and R-SVRG+, the latter shows superior performance for all the step size sequences.
In Figure \ref{fig:PCA_results}(b), while the gradient norm of SGD remains at higher values, those of R-SVRG and R-SVRG+ converge to lower values in all the cases.

\begin{figure}[htbp]
\begin{center}
	\begin{minipage}[t]{0.45\textwidth}
	\begin{center}
		\includegraphics[width=\textwidth]{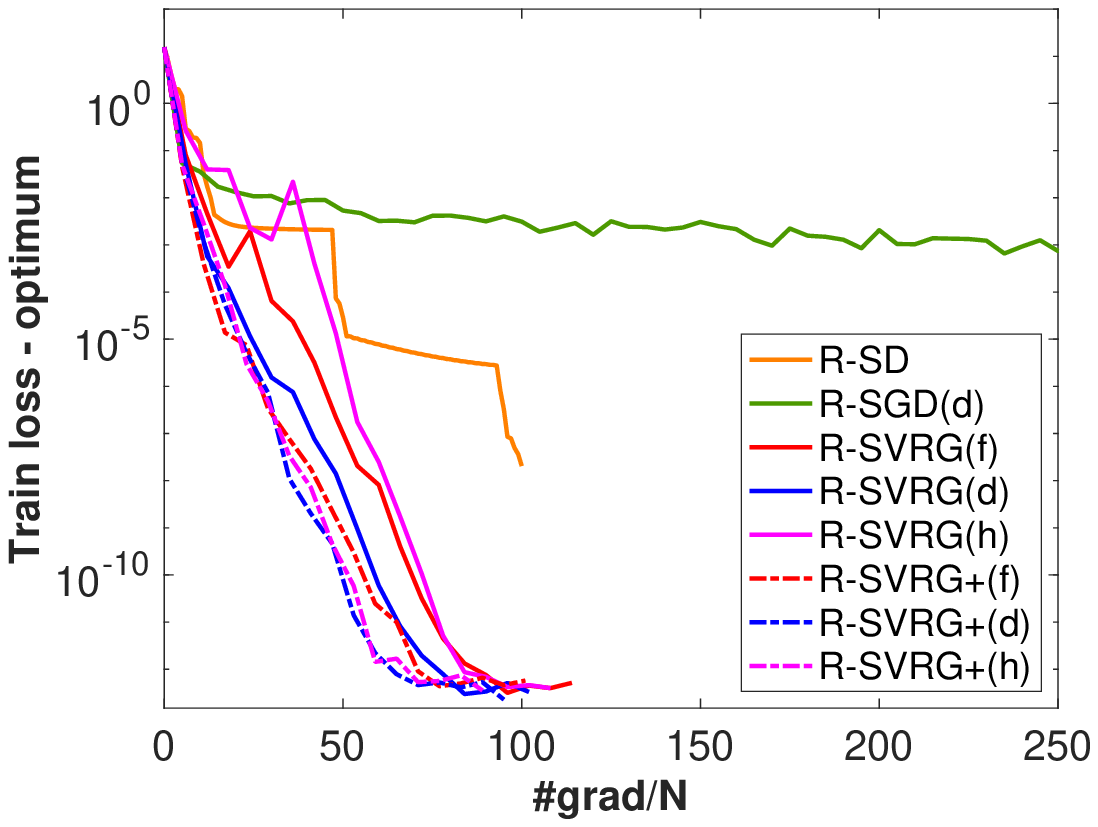}\\
		{\small (a) Optimality gap. }
	\end{center}
	\end{minipage}
	\hspace*{1.0cm}
	\begin{minipage}[t]{0.45\textwidth}
	\begin{center}
		\includegraphics[width=\textwidth]{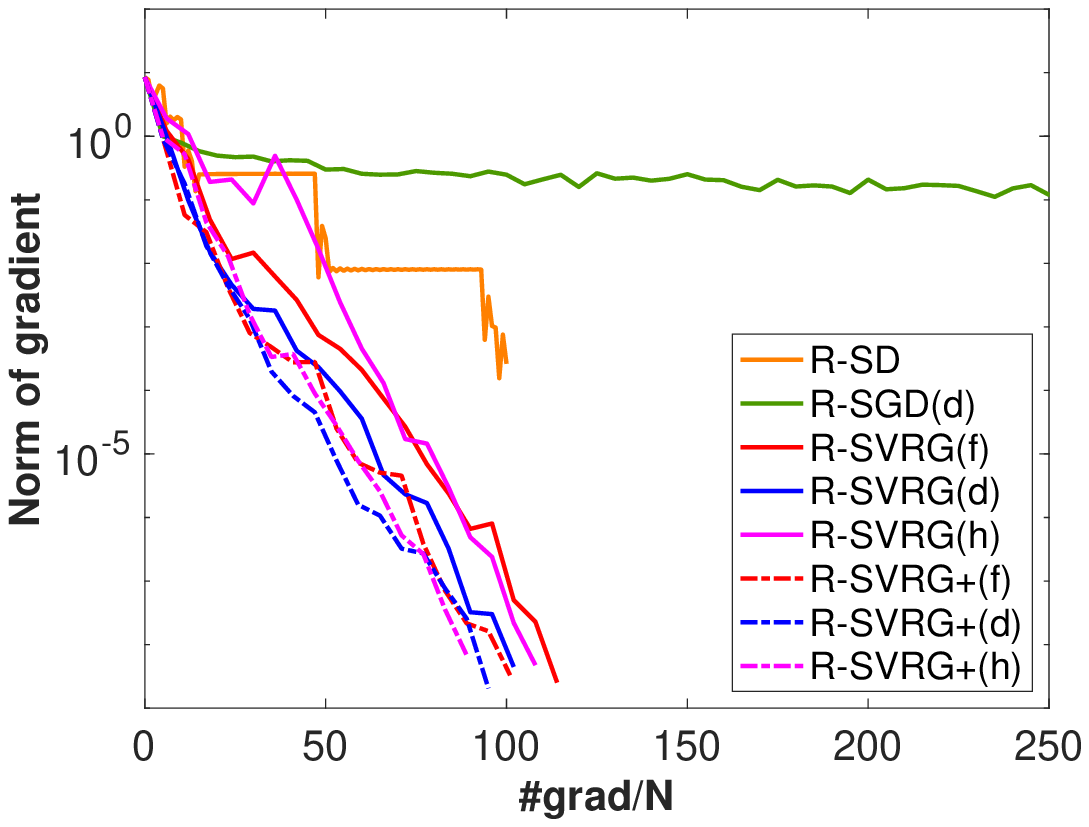}\\
		{\small (b) Norm of gradient.}
	\end{center} 
	\end{minipage}\\
\caption{Performance evaluations on PCA problem.
In the legends of the figures, (f), (d), and (h) denote fixed, decaying, and hybrid step sizes, respectively.
The parameters are chosen as follows.
R-SGD(d): $\alpha_0= 0.009$, $\lambda= 0.1$; R-SVRG(f): $\alpha_0= 0.001$; R-SVRG(d): $\alpha_0= 0.001$, $\lambda= 0.001$; R-SVRG(h): $\alpha_0=0.004$, $\lambda= 0.01$; R-SVRG+(f): $\alpha_0 = 0.001$; R-SVRG+(d): $\alpha_0= 0.002$, $\lambda= 0.01$; R-SVRG+(h): $\alpha_0 = 0.002$, $\lambda= 0.01$.}
\label{fig:PCA_results}
\end{center}		
\end{figure}

{\bf Low-rank matrix completion.} The matrix completion problem amounts to completing an incomplete matrix $\mat{X}$, say of size $d \times N$, from a small number of entries by assuming a low-rank model for the matrix. If $\Omega$ is the set of indices for which we know the entries in $\mat{X}$, the rank-$r$ matrix completion problem amounts to solving the problem
\begin{equation}\label{Eq:MC_batch}
\begin{array}{ll}
{\displaystyle \min_{{\scriptsize \mat{U}} \in \mathbb{R}^{d \times r},\ {\scriptsize \mat{A}} \in \mathbb{R}^{r \times N}}} \ \left\|\mathcal{P}_{\Omega}(\mat{UA}) - \mathcal{P}_{\Omega}(\mat X) \right\|_F^2,
\end{array}
\end{equation}
where the operator $\mathcal{P}_{\Omega}$ acts as $\mathcal{P}_{\Omega}(\mat{X}_{ij})=\mat{X}_{ij}$ if $(i,j) \in \Omega$ and $\mathcal{P}_{\Omega}(\mat{X}_{ij})=0$ otherwise.
This is called the orthogonal sampling operator and is a mathematically convenient way to represent the subset of known entries. Partitioning $\mat{X} = [\vec{x}_1, \vec{x}_2, \ldots, \vec{x}_n] $, Problem (\ref{Eq:MC_batch}) is equivalent to the problem
\begin{equation}\label{Eq:MC}
\begin{array}{lll}
{\displaystyle \min_{{\scriptsize \mat{U}} \in \mathbb{R}^{d \times r},\ \vec{a}_n \in \mathbb{R}^{r}}} 
\ 
{\displaystyle \frac{1}{N} \sum_{n=1}^N \left\|  \mathcal{P}_{\Omega_n}\left(\mat{U} \vec{a}_n\right) - \mathcal{P}_{\Omega_n}\left(\vec{x}_n\right) \right\|_2^2,
}
\end{array}
\end{equation}
where $\vec{x}_n \in \mathbb{R}^d$ and the operator $\mathcal{P}_{\Omega_n}$ is the sampling operator for the $n$-th column. Given \mat{U}, $\vec{a}_n$ in (\ref{Eq:MC}) admits a closed form solution. Consequently, Problem (\ref{Eq:MC}) depends only on the column space of $\mat{U}$ and is on the Grassmann manifold \cite{Balzano_arXiv_2010_s}.

The proposed algorithms are also compared with Grouse \cite{Balzano_arXiv_2010_s}, a state-of-the-art stochastic descent algorithm on the Grassmann manifold. We first consider a synthetic dataset with $N=5000$ and $d=500$ with rank $r=5$.
The algorithms are initialized randomly as suggested in \cite{Kressner_BIT_2014_s}.
The ten choices of $\alpha_0$ are $\{10^{-3},2\times10^{-3}, \ldots, 10^{-2}\}$ for R-SGD and R-SVRG(+), and $\{0.1,0.2, \ldots, 1.0\}$ for Grouse. This instance considers the loss on a test set $\Gamma$, which differs from training set $\Omega$. We also impose an exponential decay of the singular values. The condition number (CN) of a matrix is the ratio of the largest to the smallest singular values of the matrix. The over-sampling ratio (OS) determines the number of entries that are known. This instance uses CN=$5$ and OS=$5$. An OS of $5$ implies that $5(N+d-r)r$ randomly and uniformly selected entries are known a priori among the total $Nd$ entries.
Figure \ref{fig:MC_results}(a) shows the results of loss on the test set $\Gamma$.
These results show the superior performance of our proposed algorithms.

Next, we consider Jester dataset 1 \cite{Goldberg_IR_2001_s}, consisting of ratings of $100$ jokes by $24983$ users. Each rating is a real number from $-10$ to $10$. We randomly extract two ratings per user as the training set $\Omega$ and test set $\Gamma$. The algorithms are run by fixing the rank to $r=5$ with random initialization. $\alpha_0$ is chosen from $\{10^{-6},2\times10^{-6}, \ldots, 10^{-5}\}$ for SGD and SVRG(+) and $\{10^{-3},2\times10^{-3}, \ldots, 10^{-2}\}$ for Grouse.
Figure \ref{fig:MC_results}(b) shows the superior performance of R-SVRG(+) on the test set of the Jester dataset.

As a final test, we compare the algorithms on the MovieLens-1M dataset downloaded from \url{http://grouplens.org/datasets/movielens/}. The dataset has a million ratings corresponding to $6040$ users and $3952$ movies. $\alpha_0$ is chosen from $\{10^{-5},2\times10^{-5}, \ldots, 10^{-4}\}$.
Figure \ref{fig:MC_results}(c) shows the results on the test set for all the algorithms except Grouse, which faces issues with convergence on this data set. R-SVRG(+) shows much faster convergence than the others, and R-SVRG is better than R-SVRG+ in terms of the final test loss for all step size algorithms.

\begin{figure}[htbp]
\begin{center}
	\begin{minipage}[t]{0.45\textwidth}
	\begin{center}
		\includegraphics[width=\textwidth]{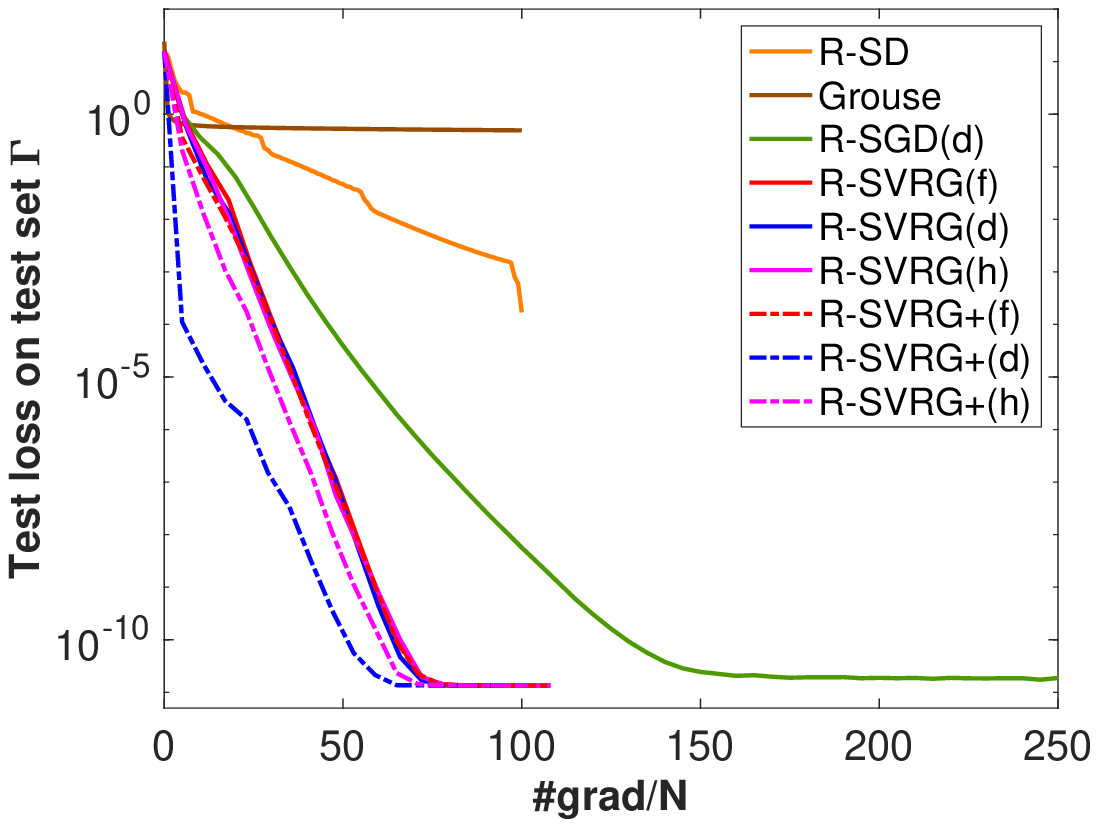}\\
		{\small (a) Test loss (synthetic).}
	\end{center} 
	\end{minipage}
	\hspace*{1.0cm}
	\begin{minipage}[t]{0.45\textwidth}
	\begin{center}
		\includegraphics[width=\textwidth]   
		{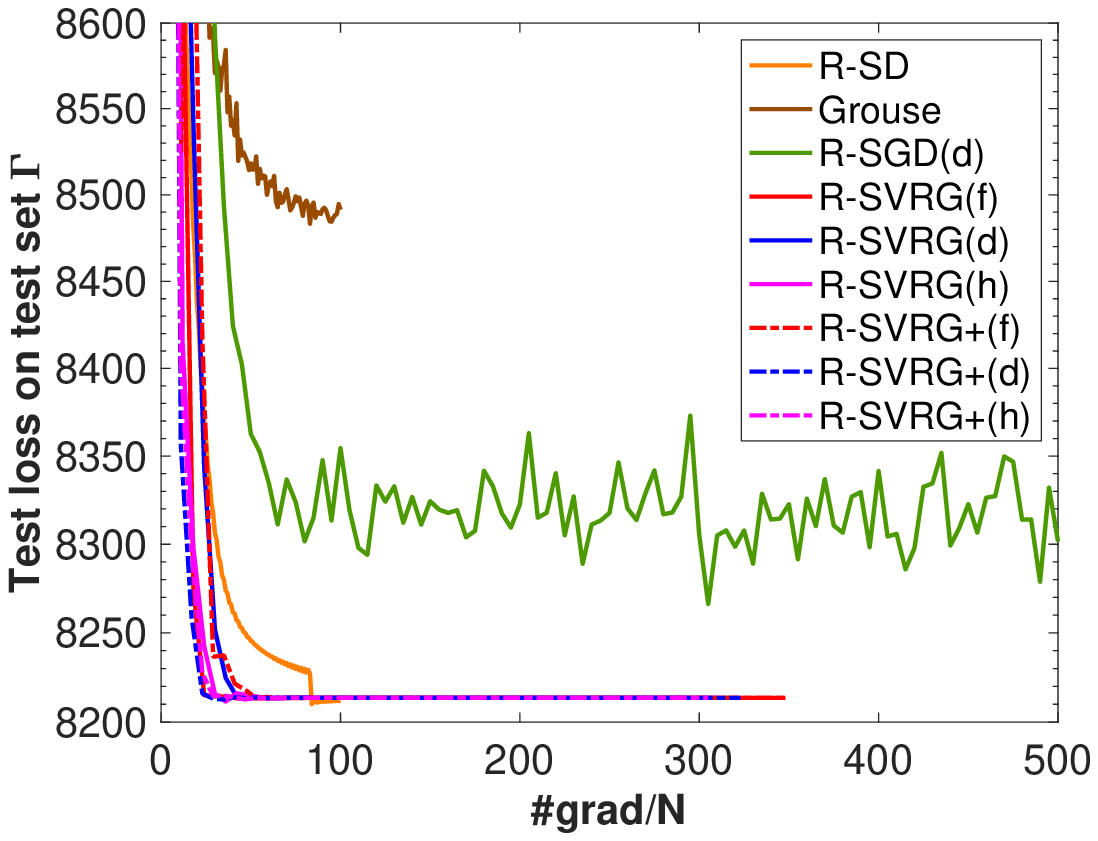}\\
		{\small (b) Test loss (Jester).}
	\end{center} 
	\end{minipage}\\
	\vspace*{0.5cm}
	
	\begin{minipage}[t]{0.45\textwidth}
	\begin{center}
		\includegraphics[width=\textwidth]{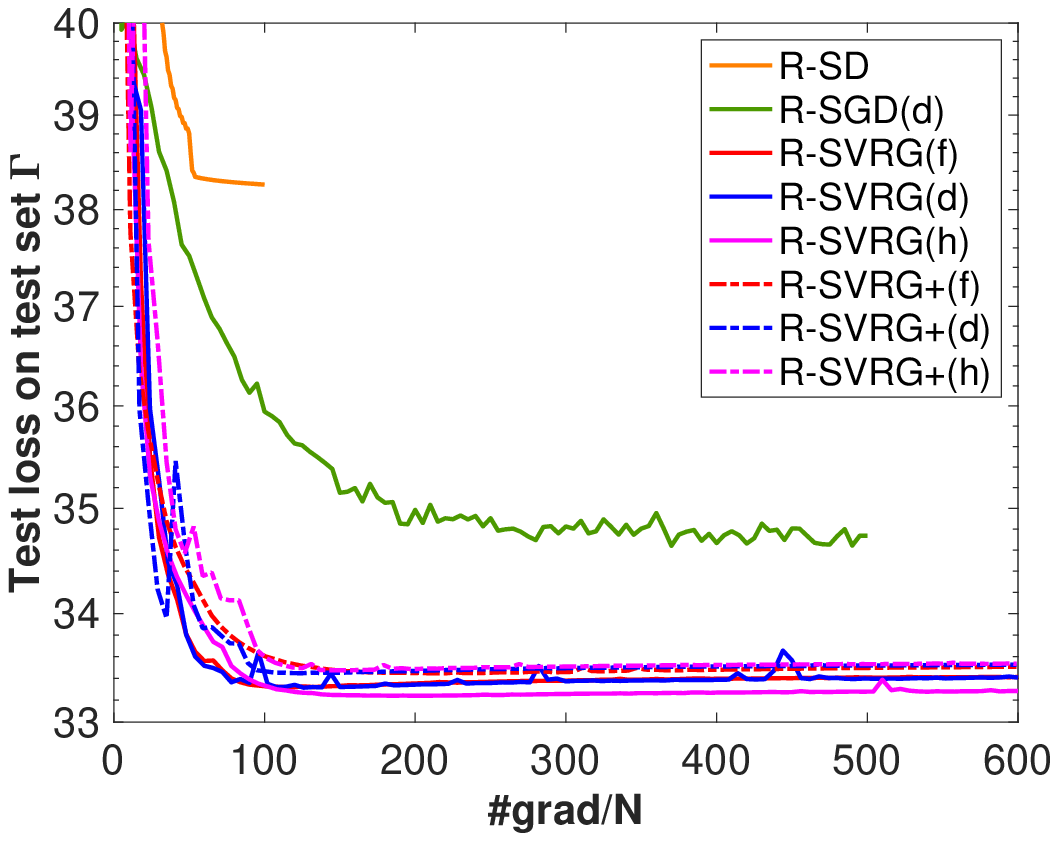}\\
		{\small (c) Test loss (MovieLens-1M).}
	\end{center} 
	\end{minipage}	
\caption{Performance evaluations on low-rank matrix completion problem.
In the legends of the figures, (f), (d), and (h) denote fixed, decaying, and hybrid step sizes, respectively.
The parameters are chosen as follows.
Grouse: (a) $\alpha_0 = 1$; (b) $\alpha_0 = 0.001$;
R-SGD(d): (a) $\alpha_0= 0.001$, $\lambda= 0.01$; (b) $\alpha_0= 10^{-6}$, $\lambda= 0.1$; (c) $\alpha_0= 10^{-5}$, $\lambda= 0.001$;
R-SVRG(f): (a) $\alpha_0= 0.002$; (b) $\alpha_0= 5 \times 10^{-6}$; (c) $\alpha_0= 5 \times 10^{-5}$;
R-SVRG(d): (a) $\alpha_0= 0.004$, $\lambda= 0.01$; (b) $\alpha_0= 7 \times 10^{-6}$, $\lambda= 0.001$; (c) $\alpha_0= 4 \times 10^{-5}$, $\lambda= 0.001$;
R-SVRG(h): (a) $\alpha_0=0.003$, $\lambda= 0.01$; (b) $\alpha_0= 6 \times 10^{-6}$, $\lambda= 0.01$; (c) $\alpha_0= 4 \times 10^{-5}$, $\lambda= 0.01$;
R-SVRG+(f): (a) $\alpha_0 = 0.002$; (b) $\alpha_0 = 6 \times 10^{-6}$; (c) $\alpha_0 = 3 \times 10^{-5}$;
R-SVRG+(d): (a) $\alpha_0= 0.01$, $\lambda= 0.01$; (b) $\alpha_0= 6 \times 10^{-6}$, $\lambda= 0.01$; (c) $\alpha_0= 5 \times 10^{-5}$, $\lambda= 0.001$;
R-SVRG+(h): (a) $\alpha_0 = 0.003$, $\lambda= 0.01$; (b) $\alpha_0= 7 \times 10^{-6}$, $\lambda= 0.001$; (c) $\alpha_0= 5 \times 10^{-5}$, $\lambda= 0.1$.
}
\label{fig:MC_results}
\end{center}	
\end{figure}

\section{Conclusion}
We proposed a Riemannian stochastic variance reduced gradient (R-SVRG) algorithm with retraction and vector transport, which includes the algorithm with exponential mapping and parallel translation as a special case. The proposed algorithm stems from the variance reduced gradient algorithm in Euclidean space, but it has been extended herein to Riemannian manifolds.
The main challenges of averaging, adding, and subtracting multiple gradients on a Riemannian manifold were handled with a vector transport. We proved that R-SVRG generates globally convergent sequences with a decaying step size and is locally linearly convergent with a fixed step size under some natural assumptions. Numerical comparisons of problems on the SPD manifold and the Grassmann manifold indicated the superior performance of R-SVRG on various benchmarks.

\appendix
\section{Proofs of lemmas}
\label{Appendix}
In this section, we present complete proofs of Lemmas \ref{AppenLem:KarcherMeanDistance}--\ref{AppenLem:UpperBoundVariance}.

\begin{proof}[Proof of Lemma \ref{AppenLem:KarcherMeanDistance}]
From the triangle inequality and $(a + b)^2 \leq 2 a^2 + 2 b^2$ for real numbers $a$ and $b$, we have for $i=1,2,\ldots,m$,
\begin{equation*}
\left({\rm dist}(p,w)\right)^2 \le \left({\rm dist}\left(p,w_i\right)+{\rm dist}\left(w_i,w\right)\right)^2 \le 2\left({\rm dist}\left(p,w_i\right)\right)^2 + 2\left({\rm dist}\left(w_i,w\right)\right)^2.
\end{equation*}
Since $w$ is the Riemannian centroid of $w_1,w_2,\ldots,w_m$, it holds that
\begin{equation*}
\sum_{i=1}^{m} \left({\rm dist}\left(w,w_i\right)\right)^2 \le \sum_{i=1}^m \left({\rm dist}\left(p,w_i\right)\right)^2.
\end{equation*}
It then follows that
\begin{equation*}
m \left({\rm dist}(p,w)\right)^2 \le 2\sum_{i=1}^m \left({\rm dist}\left(p,w_i\right)\right)^2 + 2\sum_{i=1}^m \left({\rm dist}\left(w_i,w\right)\right)^2 \le 4 \sum_{i=1}^m \left({\rm dist}\left(p,w_i\right)\right)^2.
\end{equation*}
This completes the proof.
\end{proof}

\begin{proof}[Proof of Lemma \ref{Lem_xi_norm}]
We first show that $\|\gradf_n\|$ for any $n \in \{1,2,\dots, N\}$ is upper-bounded in $\Omega$ when $\Omega$ is sufficiently small.
Since $\Omega$ is sufficiently small, it is contained in a set $U \subset \mathcal{M}$ diffeomorphic to an open set $U' \subset \mathbb{R}^{\dim \mathcal{M}}$ by a chart $\phi\colon U \to U'$.
Consider a sufficiently small closed ball $B$ in $U'$ centered at $\phi(w^*)$ such that $\phi^{-1}(B) \subset U$.
Then, $w^*$ is in $\phi^{-1}(B)$.
Note that $B$ is compact and $\phi$ is a diffeomorphism.
Hence, $\phi^{-1}(B)$ is also compact.
Replacing $\Omega$ with this $\phi^{-1}(B)$ if necessary, we can assume that $\Omega$ is compact.
Therefore, $\|\gradf_n\|$ is upper bounded in $\Omega$.
In this proof, let $C$ denote a constant such that $\|\gradf_n(z)\|_z \le C$ for all $n \in \{1,2,\dots,N\}$ and $z \in \Omega$.

Let $\tilde{\eta}_{t-1}^s$ satisfy $R_{\tilde{w}^{s-1}}(\tilde{\eta}_{t-1}^s) = w_{t-1}^s$.
The definition of $\xi_t^s$, Eq.~\eqref{Eq:R-SVRG-Grad-paralleltrans}, and the triangle inequality yield
\begin{align*}
&\left\|\xi_t^s\right\|_{w_{t-1}^s} \\
= & \left\| \gradf_{i_t^s}\left(w_{t-1}^s\right) - \mathcal{T}_{\tilde{\eta}_{t-1}^s}\left(\gradf_{i_t^s}\left(\tilde{w}^{s-1}\right)-\gradf\left(\tilde{w}^{s-1}\right)\right)\right\|_{w_{t-1}^s}\\
\le & \left\| \gradf_{i_t^s}\left(w_{t-1}^s\right) - \mathcal{T}_{\tilde{\eta}_{t-1}^s}\left(\gradf_{i_t^s}\left(\tilde{w}^{s-1}\right)\right) \right\|_{w_{t-1}^s} + \left\| \mathcal{T}_{\tilde{\eta}_{t-1}^s}\left(\gradf\left(\tilde{w}^{s-1}\right)\right)\right\|_{w_{t-1}^s} \\
= & \left\| \gradf_{i_t^s}\left(w_{t-1}^s\right) - P^{w_{t-1}^s \leftarrow \tilde{w}^{s-1}}\left(\gradf_{i_t^s}\left(\tilde{w}^{s-1}\right)\right) \right\|_{w_{t-1}^s}\\
& + \left\|P^{w_{t-1}^s \leftarrow \tilde{w}^{s-1}}\left(\gradf_{i_t^s}\left(\tilde{w}^{s-1}\right)\right) - \mathcal{T}_{\tilde{\eta}_{t-1}^s}\left(\gradf_{i_t^s}\left(\tilde{w}^{s-1}\right)\right)\right\|_{w_{t-1}^s} \\
& + \left\| \gradf\left(\tilde{w}^{s-1}\right)\right\|_{\tilde{w}^{s-1}},
\end{align*}
where $\gamma_{t-1}^s(\tau) = R_{\tilde{w}^{s-1}}(\tau\tilde{\eta}_{t-1}^s)$ and $P$ is the parallel translation operator along curve $\gamma_{t-1}^s$.
The three terms on the right-hand side above are bounded as follows.
For the first term, it follows from Lemma \ref{Lem:pseudo_Lipschitz} that
\begin{equation*}
\left\| \gradf_{i_t^s}\left(w_{t-1}^s\right) - P^{w_{t-1}^s \leftarrow \tilde{w}^{s-1}}\left(\gradf_{i_t^s}\left(\tilde{w}^{s-1}\right)\right) \right\|_{w_{t-1}^s} \le \beta {\rm dist}\left(w_{t-1}^s, \tilde{w}^{s-1}\right) \le 2\beta r,
\end{equation*}
where $\beta$ is in the lemma. The second term can be evaluated from Lemmas \ref{Lem:T_Lipschitz} and \ref{lemma:retraction_dist} as
\begin{align*}
& \left\|P^{w_{t-1}^s \leftarrow \tilde{w}^{s-1}}\left(\gradf_{i_t^s}\left(\tilde{w}^{s-1}\right)\right) - \mathcal{T}_{\tilde{\eta}_{t-1}^s}\left(\gradf_{i_t^s}\left(\tilde{w}^{s-1}\right)\right)\right\|_{w_{t-1}^s} \\
\le & \theta\left\|\gradf_{i_t^s}\left(\tilde{w}^{s-1}\right)\right\|_{\tilde{w}^{s-1}} \left\| \tilde{\eta}_{t-1}^s\right\|_{\tilde{w}^{s-1}}
\le \theta C\mu{\rm dist}\left(w_{t-1}^s, \tilde{w}^{s-1}\right) \le 2\theta C\mu r,
\end{align*}
where $\theta$ and $\mu$ are in the lemmas. From Lemma \ref{Lem:pseudo_Lipschitz}, we have
\begin{align*}
\left\| \gradf\left(\tilde{w}^{s-1}\right)\right\|_{\tilde{w}^{s-1}}
= &\left\| \gradf\left(\tilde{w}^{s-1}\right) - P_{*}^{\tilde{w}^{s-1} \leftarrow w^{*}}\left(\gradf\left({w}^{*}\right)\right) \right\|_{\tilde{w}^{s-1}}\\
 \le & \beta {\rm dist}\left(\tilde{w}^{s-1}, w^{*}\right) \le \beta r
\end{align*}
for the third term, where $P_*$ is the parallel translation along curve $\tilde{\gamma}_*^{s-1}$ defined by $\tilde{\gamma}_*^{s-1}(\tau) = R_{w^*}(\tau \tilde{\eta}_*^{s-1})$ with $\tilde{\eta}_*^{s-1}$ satisfying $R_{w^*}(\tilde{\eta}_*^{s-1}) = \tilde{w}^{s-1}$.
Therefore, we have
\begin{equation*}
\left\|\xi_t^s\right\|_{w_{t-1}^s} \le r(3\beta+2\theta C\mu) < \varepsilon
\end{equation*}
if we choose a sufficiently small $r$ such that $r<\varepsilon/(3\beta+2\theta C\mu)$.
\end{proof}

\begin{proof}[Proof of Lemma \ref{AppenLem:UpperBoundVariance}]
Let $\eta_{t-1}^{*s} \in T_{w^*} \mathcal{M}$ and $\tilde{\eta}_{t-1}^s \in T_{\tilde{w}^{s-1}} \mathcal{M}$ satisfy $R_{w^*}(\eta_{t-1}^{*s}) = w_{t-1}^s$ and $R_{\tilde{w}^{s-1}}(\tilde{\eta}_{t-1}^s)=w_{t-1}^s$, respectively.
Let $P^{w \leftarrow z}$ be a parallel translation along the curve $R_z(\tau\eta)$, where $R_z(\eta)=w$.
By Lemmas \ref{Lem:pseudo_Lipschitz} and \ref{Lem:T_Lipschitz},
the upper bound of $\mathbb{E}_{i_t^s}[\| \xi_t^s \|_{w_{t-1}^s}^2]$ in terms of the distance of $w_{t-1}^s$ and $\tilde{w}^{s-1}$ from $w^*$ is computed as
\allowdisplaybreaks[1]
\begin{align*}
& \mathbb{E}_{i_t^s}\left[\left\| \xi_t^s \right\|_{w_{t-1}^s}^2\right] \nonumber\\
=&\mathbb{E}_{i_t^s}\left[\Bigl\| \left( \gradf_{i_t^s}\left(w_{t-1}^{s}\right) - \mathcal{T}_{\eta_{t-1}^{*s}}\left(\gradf_{i_t^s}\left(w^{*}\right)\right) \right)  \right. \nonumber\\
& \left. + \left( \mathcal{T}_{\eta_{t-1}^{*s}}\left(\gradf_{i_t^s}\left(w^{*}\right)\right) -  \mathcal{T}_{\tilde{\eta}_{t-1}^s} \left(\gradf_{i_t^s}\left(\tilde{w}^{s-1}\right) \right) + \mathcal{T}_{\tilde{\eta}_{t-1}^s} \left(\gradf\left(\tilde{w}^{s-1}\right) \right)\right) \Bigr\|_{w_{t-1}^s}^2 \right] \nonumber\\
\leq & 2\mathbb{E}_{i_t^s}\left[ \left\| \gradf_{i_t^s}\left(w_{t-1}^{s}\right) - \mathcal{T}_{\eta_{t-1}^{*s}}\left(\gradf_{i_t^s}\left(w^{*}\right)\right) \right\|_{w_{t-1}^s}^2 \right]  \nonumber\\
& + 2\mathbb{E}_{i_t^s}\left[ \left\| \mathcal{T}_{\tilde{\eta}_{t-1}^s} \left(\gradf_{i_t^s}\left(\tilde{w}^{s-1}\right) \right) - \mathcal{T}_{\eta_{t-1}^{*s}}\left(\gradf_{i_t^s}\left(w^{*}\right)\right) - \mathcal{T}_{\tilde{\eta}_{t-1}^s} \left(\gradf\left(\tilde{w}^{s-1}\right) \right) \right\|_{w_{t-1}^s}^2 \right]  \nonumber\\
= & 2\mathbb{E}_{i_t^s}\left[ \left\| \gradf_{i_t^s}\left(w_{t-1}^{s}\right) - \mathcal{T}_{\eta_{t-1}^{*s}}\left(\gradf_{i_t^s}\left(w^{*}\right)\right) \right\|_{w_{t-1}^s}^2 \right]  \nonumber\\
&+ 2\mathbb{E}_{i_t^s}\left[ \left\| \mathcal{T}_{\tilde{\eta}_{t-1}^s} \left(\gradf_{i_t^s}\left(\tilde{w}^{s-1}\right) \right) - \mathcal{T}_{\eta_{t-1}^{*s}}\left(\gradf_{i_t^s}\left(w^{*}\right)\right) \right\|_{w_{t-1}^s}^2 \right]  \nonumber\\
& - 4 \left\langle \mathcal{T}_{\tilde{\eta}_{t-1}^s} \left(\gradf\left(\tilde{w}^{s-1}\right) \right), \mathcal{T}_{\tilde{\eta}_{t-1}^s} \left(\gradf\left(\tilde{w}^{s-1}\right) \right) - \mathcal{T}_{\eta_{t-1}^{*s}}\left(\gradf\left(w^{*}\right)\right)\right\rangle_{w_{t-1}^s}  \nonumber\\
&+ 2 \| \gradf(\tilde{w}^{s-1}) \|_{\tilde{w}^{s-1}}^2 \nonumber\\
= & 2\mathbb{E}_{i_t^s}\biggl[ \left\| \gradf_{i_t^s}\left(w_{t-1}^{s}\right) - P^{w_{t-1}^{s} \leftarrow w^{*}}\left(\gradf_{i_t^s}\left(w^{*}\right)\right)\right.\nonumber\\
& \left.\qquad \qquad \qquad \qquad \qquad + P^{w_{t-1}^{s} \leftarrow w^{*}}\left(\gradf_{i_t^s}\left(w^{*}\right)\right) - \mathcal{T}_{\eta_{t-1}^{*s}}\left(\gradf_{i_t^s}\left(w^{*}\right)\right) \right\|_{w_{t-1}^s}^2 \biggr]  \nonumber\\
&+ 2\mathbb{E}_{i_t^s}\biggl[ \left\| \mathcal{T}_{\tilde{\eta}_{t-1}^s} \left(\gradf_{i_t^s}\left(\tilde{w}^{s-1}\right) \right) - \gradf_{i_t^s}\left(w_{t-1}^{s}\right) \right. \nonumber\\
& \left. \qquad \qquad \qquad \qquad \qquad \qquad \qquad + \gradf_{i_t^s}\left(w_{t-1}^{s}\right)- \mathcal{T}_{\eta_{t-1}^{*s}}\left(\gradf_{i_t^s}\left(w^{*}\right)\right) \right\|_{w_{t-1}^s}^2 \biggr]  \nonumber\\
& - 2 \left\| \gradf\left(\tilde{w}^{s-1}\right) \right\|_{w^{s-1}}^2 \nonumber\\
\leq & 4\mathbb{E}_{i_t^s}\left[ \left\| \gradf_{i_t^s}\left(w_{t-1}^{s}\right) - P^{w_{t-1}^{s} \leftarrow w^{*}}\left(\gradf_{i_t^s}\left(w^{*}\right)\right)  \right\|_{w_{t-1}^s}^2 \right] \nonumber\\
&+ 4\mathbb{E}_{i_t^s}\left[ \left\| P^{w_{t-1}^{s} \leftarrow w^{*}}\left(\gradf_{i_t^s}\left(w^{*}\right)\right) - \mathcal{T}_{\eta_{t-1}^{*s}}\left(\gradf_{i_t^s}\left(w^{*}\right)\right) \right\|_{w_{t-1}^s}^2 \right]  \nonumber\\
&+ 4\mathbb{E}_{i_t^s}\left[ \left\| \mathcal{T}_{\tilde{\eta}_{t-1}^s} \left(\gradf_{i_t^s}\left(\tilde{w}^{s-1}\right) \right) - \gradf_{i_t^s}\left(w_{t-1}^{s}\right)\right\|_{w_{t-1}^s}^2 \right]
\nonumber\\
&
 + 4\mathbb{E}_{i_t^s}\left[ \left\| \gradf_{i_t^s}\left(w_{t-1}^{s}\right) - \mathcal{T}_{\eta_{t-1}^{*s}}\left(\gradf_{i_t^s} \left(w^{*}\right)\right) \right\|_{w_{t-1}^s}^2 \right] \nonumber\\
\le & 
4\beta^2\left({\rm dist}\left(w^{s}_{t-1}, w^*\right)\right)^2 + 4\theta^2\left\| \eta_{t-1}^{*s}\right\|_{w^*}^2\mathbb{E}_{i_t^s}\left[\left\|\gradf_{i_t^s}\left(w^{*}\right)\right\|_{w^*}^2\right]\nonumber\\
&
+ 4\mathbb{E}_{i_t^s}\biggl[ \left\| \mathcal{T}_{\tilde{\eta}_{t-1}^s} \left(\gradf_{i_t^s}\left(\tilde{w}^{s-1}\right) \right) - P^{w_{t-1}^s \leftarrow \tilde{w}^{s-1}}\left(\gradf_{i_t^s}\left(\tilde{w}^{s-1}\right)\right) \right.\nonumber\\
&  \qquad \qquad \qquad \qquad \qquad  \left.+ P^{w_{t-1}^s \leftarrow \tilde{w}^{s-1}}\left(\gradf_{i_t^s}\left(\tilde{w}^{s-1}\right)\right) -\gradf_{i_t^s}\left(w_{t-1}^{s}\right)\right\|_{w_{t-1}^s}^2 \biggr]
\nonumber\\
&
 + 2\left(4\beta^2\left({\rm dist}\left(w^{s}_{t-1}, w^*\right)\right)^2 + 4\theta^2\left\| \eta_{t-1}^{*s}\right\|_{w^*}^2\mathbb{E}_{i_t^s}\left[\|\gradf_{i_t^s}\left(w^{*}\right)\right\|_{w^*}^2\right]) \nonumber
\\
\leq & 12\left(\beta^2\left({\rm dist}\left(w^{s}_{t-1}, w^*\right)\right)^2+C^2\theta^2\left\| \eta^{*s}_{t-1}\right\|_{w^*}^2\right) \nonumber\\
& + 8\theta^2\mathbb{E}_{i_t^s}\left[\left\|\tilde{\eta}_{t-1}^s\right\|_{\tilde{w}^{s-1}}^2\left\|\gradf_{i_t^s}\left(\tilde{w}^{s-1}\right)\right\|_{\tilde{w}^{s-1}}^2\right] + 8\beta^2\left({\rm dist}\left(\tilde{w}^{s-1}, w_{t-1}^s\right)\right)^2\nonumber\\
\le &
4\left(\beta^2+\mu^2C^2\theta^2\right)\left(3\left({\rm dist}\left(w_{t-1}^s,w^*\right)\right)^2 + 2\left({\rm dist}\left(\tilde{w}^{s-1},w_{t-1}^s\right)\right)^2\right)
 \nonumber\\
\le & 
4\left(\beta^2+\mu^2C^2\theta^2\right)\left(3\left({\rm dist}\left(w_{t-1}^s,w^*\right)\right)^2 + 2\left({\rm dist}\left(\tilde{w}^{s-1},w^*\right)+{\rm dist}\left(w^*,w_{t-1}^s\right)\right)^2\right)
 \nonumber\\
\le & 
4\left(\beta^2+\mu^2C^2\theta^2\right)\left(7\left({\rm dist}\left(w_{t-1}^s,w^*\right)\right)^2 + 4\left({\rm dist}\left(\tilde{w}^{s-1},w^*\right)\right)^2\right),
\end{align*}
where the relation $\| a + b\|^2 \leq 2 \|a\|^2 + 2 \|b\|^2$ for vectors $a$ and $b$ in a norm space and the triangle inequality are used repeatedly.
Note also that $\mathbb{E}_{i_t^s}[\gradf_{i_t^s}(\tilde{w}^{s-1})]=\gradf(\tilde{w}^{s-1})$ and $\gradf(w^*)=0$ and that $\mathbb{E}_{i_t^s}$ is a linear operator.
Furthermore, we have evaluated the value $\mathbb{E}_{i_t^s}[\|\gradf_{i_t^s}(w_{t-1}^s) - \mathcal{T}_{\eta_{t-1}^{*s}}(\gradf_{i_t^s}(w^*))\|_{w_{t-1}^s}^2$ and again used the obtained relation in the third inequality.
\end{proof}


\begin{thebibliography}{99}
\bibitem{Absil_OptAlgMatManifold_2008}
P.-A. Absil, R. Mahony, and R. Sepulchre, Optimization Algorithms on Matrix Manifolds,
Princeton University Press, Princeton, NJ, 2008.

\bibitem{allen2016variance}
Z. Allen-Zhu and E. Hazan, Variance reduction for faster non-convex optimization, in Proceedings of the 33rd International Conference on Machine Learning, Proc. Mach. Learn.
Res. 48, 2016, pp. 699--707; available at http://proceedings.mlr.press/v48/.

\bibitem{allen2016improved}
Z. Allen-Zhu and Y. Yuan, Improved SVRG for non-strongly-convex or sum-of-non-convex
objectives, in Proceedings of the 33rd International Conference on Machine Learning, Proc.
Mach. Learn. Res. 48, 2016, pp. 1080--1089; available at http://proceedings.mlr.press/v48/.

\bibitem{Balzano_arXiv_2010_s}
L. Balzano, R. Nowak, and B. Recht, Online identification and tracking of subspaces from
highly incomplete information, in Proceedings of the 48th Annual Allerton Conference on
Communication, Control, and Computing, IEEE Press, Piscataway, NJ, 2010, pp. 704--711.

\bibitem{Bonnabel_IEEETAC_2013_s}
S. Bonnabel, Stochastic gradient descent on Riemannian manifolds, IEEE Trans. Automat.
Control, 58 (2013), pp. 2217--2229.

\bibitem{Boumal_Manopt_2014_s}
N. Boumal, B. Mishra, P.-A. Absil, and R. Sepulchre, Manopt: A MATLAB toolbox for
optimization on manifolds, J. Mach. Learn. Res., 15 (2014), pp. 1455--1459.

\bibitem{Defazio_NIPS_2014_s}
A. Defazio, F. Bach, and S. Lacoste-Julien, SAGA: A fast incremental gradient method
with support for non-strongly convex composite objectives, in Adv. Neural Inf. Process.
Syst. 27, Curran Associates, Red Hook, NY, 2014, pp. 1646--1654.

\bibitem{Fisk_1965_s}
D. L. Fisk, Quasi-martingales, Trans. Amer. Math. Soc., 120 (1965), pp. 369--389.

\bibitem{Garber_arXiv_2015_s}
D. Garber and E. Hazan, Fast and Simple PCA via Convex Optimization, preprint, https://arxiv.org/abs/1509.05647, 2015.

\bibitem{Goldberg_IR_2001_s}
K. Goldberg, T. Roeder, D. Gupta, and C. Perkins, Eigentaste: A constant time collaborative filtering algorithm, Inf. Retr., 4 (2001), pp. 133--151.

\bibitem{huang2015riemannian}
W. Huang, P.-A. Absil, and K. A. Gallivan, A Riemannian symmetric rank-one trust-region
method, Math. Program., 150 (2015), pp. 179--216.

\bibitem{huang2015broyden}
W. Huang, K. A. Gallivan, and P.-A. Absil, A Broyden class of quasi-Newton methods for
Riemannian optimization, SIAM J. Optim., 25 (2015), pp. 1660--1685.

\bibitem{Jayasumana_IEEETPAMI_2015_s}
S. Jayasumana, R. Hartley, M. Salzmann, H. Li, and M. Harandi, Kernel methods on
Riemannian manifolds with Gaussian RBF kernels, IEEE Trans. Pattern Anal. Mach.
Intell., 37 (2015), pp. 2464--2477.

\bibitem{JeurisVV_2012}
B. Jeuris, R. Vandebril, and B. Vandereycken, A survey and comparison of contemporary
algorithms for computing the matrix geometric mean, Electron. Trans. Numer. Anal., 39
(2012), pp. 379--402.

\bibitem{Johnson_NIPS_2013_s}
R. Johnson and T. Zhang, Accelerating stochastic gradient descent using predictive variance
reduction, in Adv. Neural Inf. Process. Syst. 26, Curran Associates, Red Hook, NY, 2013,
pp. 315--323.

\bibitem{Kasai_arXiv_2016}
H. Kasai, H. Sato, and B. Mishra, Riemannian Stochastic Variance Reduced Gradient on
Grassmann Manifold, preprint, https://arxiv.org/abs/1605.07367, 2016.

\bibitem{pmlr-v80-kasai18a}
H. Kasai, H. Sato, and B. Mishra, Riemannian stochastic recursive gradient algorithm,
in Proceedings of the 35th International Conference on Machine Learning, J. Dy and
A. Krause, eds., Proc. Mach. Learn. Res. 80, 2018, pp. 2516--2524; available at http:
//proceedings.mlr.press/v80/.

\bibitem{pmlr-v84-kasai18a}
H. Kasai, H. Sato, and B. Mishra, Riemannian stochastic quasi-Newton algorithm with
variance reduction and its convergence analysis, in Proceedings of the Twenty-First International Conference on Artificial Intelligence and Statistics, A. Storkey and F. Perez-Cruz,
eds., Proc. Mach. Learn. Res. 84, 2018, pp. 269--278; available at http://proceedings.mlr.
press/v84/.

\bibitem{Konecny_arXiv_2013}
J. Kone\v{c}n\'y and P. Richt\'arik, Semi-Stochastic Gradient Descent Methods, preprint, https:
//arxiv.org/abs/1312.1666, 2013.

\bibitem{Kressner_BIT_2014_s}
D. Kressner, M. Steinlechner, and B. Vandereycken, Low-rank tensor completion by
Riemannian optimization, BIT, 54 (2014), pp. 447--468.

\bibitem{Mairal_SIAMJOPT_2015}
J. Mairal, Incremental majorization-minimization optimization with application to large-scale
machine learning, SIAM J. Optim., 25 (2015), pp. 829--855.

\bibitem{Meyer_ICML_2011}
G. Meyer, S. Bonnabel, and R. Sepulchre, Linear regression under fixed-rank constraints:
A Riemannian approach, in Proceedings of the 28th International Conference on Machine
Learning, 2011, Omnipress, Madison, WI, pp. 545--552.

\bibitem{Mishra_ICDC_2014_s}
B. Mishra and R. Sepulchre, R3MC: A Riemannian three-factor algorithm for low-rank
matrix completion, in Proceedings of the 53rd IEEE Conference on Decision and Control,
IEEE Press, Piscataway, NJ, 2014, pp. 1137--1142.

\bibitem{Mishra_SIAMJOPT_2016}
B. Mishra and R. Sepulchre, Riemannian preconditioning, SIAM J. Optim., 26 (2016),
pp. 635--660.

\bibitem{Pennec_IJCV_2006}
X. Pennec, P. Fillard, and N. Ayache, A Riemannian framework for tensor computing, Int.
J. Comput. Vis., 66 (2006), pp. 41--66.

\bibitem{reddi2016stochastic}
S. J. Reddi, A. Hefny, S. Sra, B. Poczos, and A. Smola, Stochastic variance reduction for
nonconvex optimization, in Proceedings of the 33rd International Conference on Machine
Learning, Proc. Mach. Learn. Res. 48, 2016, pp. 314--323; available at http://proceedings.
mlr.press/v48/.

\bibitem{Roux_NIPS_2012_s}
N. L. Roux, M. Schmidt, and F. R. Bach, A stochastic gradient method with an exponential
convergence rate for finite training sets, in Adv. Neural Inf. Process. Syst. 25, Curran
Associates, Red Hook, NY, 2012, pp. 2663--2671.

\bibitem{Shalev_arXiv_2015_s}
S. Shalev-Shwartz, SDCA without Duality, preprint, https://arxiv.org/abs/1502.06177, 2015.

\bibitem{Shalev_arXiv_2012_s}
S. Shalev-Shwartz and T. Zhang, Proximal Stochastic Dual Coordinate Ascent, preprint,
https://arxiv.org/abs/1211.2717, 2012.

\bibitem{Shalev_JMLR_2013_s}
S. Shalev-Shwartz and T. Zhang, Stochastic dual coordinate ascent methods for regularized
loss minimization, J. Mach. Learn. Res., 14 (2013), pp. 567--599.

\bibitem{Shamir_arXiv_2015_s}
O. Shamir, Fast Stochastic Algorithms for SVD and PCA: Convergence Properties and Convexity, preprint, https://arxiv.org/abs/1507.08788, 2015.

\bibitem{Zhang_SIAMJO_2014_s}
L. Xiao and Y. Zhang, Change to paper A proximal stochastic gradient method with progressive variance reduction, SIAM J. Optim., 24 (2014), pp. 2057--2075.

\bibitem{Yuana_ICCS_2016_s}
X. Yuan, P.-A. Huang, W. Absil, and K. A. Gallivan, A Riemannian limited-memory
BFGS algorithm for computing the matrix geometric mean, in Proceedings of the International Conference on Computational Science, Procedia Comput. Sci. 80 (2016), pp. 2147--2157.

\bibitem{Zhang_NIPS_2016}
H. Zhang, S. J. Reddi, and S. Sra, Riemannian SVRG: Fast stochastic optimization on
Riemannian manifolds, in Adv. Neural Inf. Process. Syst. 29, Curran Associates, Red
Hook, NY, 2016, pp. 4592--4600.

\bibitem{Zhang_JMLR_2016_s}
H. Zhang and S. Sra, First-order methods for geodesically convex optimization, in Conference
on Learning Theory 2016, Proc. Mach. Learn. Res. 49, 2016, pp. 1617--1638; available at
http://proceedings.mlr.press/v49/.
\end{thebibliography}
\end{document}